\PassOptionsToPackage{table}{xcolor}
\documentclass[sigconf,nonacm]{acmart}
\AtBeginDocument{%
  }

\usepackage{booktabs}
\usepackage{multirow}
\usepackage{tabularx} 
\usepackage{threeparttable} 
\usepackage{xcolor} 
\usepackage{svg}
\usepackage{subcaption}
\usepackage{enumitem} 
\usepackage{algorithm}
\usepackage{algpseudocode}
\usepackage{caption}

\usepackage[utf8]{inputenc}
\usepackage[T1]{fontenc}

\begin{document}

\title{D$^3$-Subsidy: Online and Sequential Driver Subsidy Decision-Making for Large-Scale Ride-Hailing Market}
\titlenote{Accepted at the 32nd ACM SIGKDD Conference on Knowledge Discovery and Data Mining (KDD 2026). This is an arXiv author version.}

\author{Taijie Chen}
\email{ctj21@connect.hku.hk}
\orcid{0000-0003-4504-9994}
\affiliation{%
  \institution{University of Hong Kong}
  \city{Hong Kong}
  \country{China}
}

\author{Rui Su}
\email{2022111220@stu.hit.edu.cn}
\orcid{0009-0002-9282-4586}
\affiliation{%
  \institution{Harbin Institute of Technology}
  \city{Harbin}
  \state{Heilongjiang}
  \country{China}
}

\author{Siyuan Feng}
\authornote{Corresponding author}
\email{siyuan.feng@polyu.edu.hk}
\affiliation{%
  \institution{Hong Kong Polytechnic University}
  \city{Hong Kong}
  \country{China}
}

\author{Laoming Zhang}
\email{adamzhang@didiglobal.com}
\affiliation{%
  \institution{Didi International Business Group}
  \city{Beijing}
  \country{China}
}

\author{Hongyang Zhang}
\email{hong-yang.zhang@connect.polyu.hk}
\affiliation{%
  \institution{Hong Kong Polytechnic University}
  \city{Hong Kong}
  \country{China}
}

\author{Haijiao Wang}
\email{wanghaijiao@didiglobal.com}
\affiliation{%
  \institution{Didi International Business Group}
  \city{Beijing}
  \country{China}
}

\author{Zhaofeng Ma}
\email{mazhaofeng@didiglobal.com}
\affiliation{%
  \institution{Didi International Business Group}
  \city{Beijing}
  \country{China}
}

\author{Jintao Ke}
\authornotemark[1]
\email{kejintao@hku.hk}
\affiliation{%
  \institution{University of Hong Kong}
  \city{Hong Kong}
  \country{China}
}

\author{Li Ma}
\email{malimarey@didiglobal.com}
\affiliation{%
  \institution{Didi International Business Group}
  \city{Beijing}
  \country{China}
}

\renewcommand{\shortauthors}{Taijie Chen et al.}

\begin{abstract}
Ride-hailing platforms like Didi Chuxing operate in highly dynamic environments where balancing driver supply and passenger demand is critical. Although driver-side subsidies serve as a primary lever to align these forces and improve key KPIs like completed rides (\texttt{Rides}) and gross merchandise value (\texttt{GMV}), optimizing them in production requires simultaneously meeting three constraints: (i) responsiveness to stochastic shocks, (ii) strict subsidy-rate caps, and (iii) low-latency execution at city scale. These requirements rule out expensive per-order optimization, calling for a forward-looking, constraint-aware city-level controller for online sequential decision making.
To meet these requirements, we introduce D$^3$-Subsidy (Dynamic Driver-side Diffusion-based Subsidy), a hierarchical diffusion-based framework for deployable city-wide subsidy control. To bridge the train-inference gap, D$^3$-Subsidy employs a prefix-conditioned diffusion model that samples plausible future trajectories from immutable historical observations, ensuring the training protocol aligns with the fixed-history nature of online deployment. These generated plans are then decoded by a context-conditioned inverse module into low-dimensional city-level control signals. For scalable execution, we bridge the gap between city-level planning and fine-grained dispatch via a Lagrangian-dual-derived mapping, which embeds subsidy-rate caps directly into order-driver incentives without iterative optimization. Additionally, a multi-city pretraining strategy with parameter-efficient fine-tuning enables robust transfer across heterogeneous cities. Extensive offline evaluations demonstrate that D$^3$-Subsidy improves \texttt{Rides} and \texttt{GMV} while enhancing cap compliance, and a real-world A/B test confirms significant uplift while keeping budget-related violation metrics within operational thresholds.
\end{abstract}




\maketitle

\section{Introduction}
\label{sec:intro}
Ride-hailing platforms such as Uber and Didi Chuxing are a core part of urban mobility, matching millions of passengers and drivers in rapidly changing spatiotemporal environments~\cite{qin2022reinforcement}. A persistent operational challenge is ensuring sufficient driver participation: due to heterogeneous preferences and opportunity costs, many drivers are reluctant to accept low-attractiveness requests, which can trigger severe supply--demand imbalance and degraded service quality~\cite{chen2025grab}. Driver-side subsidies are therefore widely used to stimulate acceptance and completion, improving platform-level objectives such as completed rides (\texttt{Rides}) and gross merchandise value (\texttt{GMV})~\cite{yuan2021real}.

Subsidies, however, must be allocated under strict profitability constraints. In practice, platforms impose a global budget over a long horizon (e.g., one day), often expressed as a cap on the subsidy rate relative to \texttt{GMV}. This makes subsidy allocation inherently online and sequential: a decision at the current time window affects not only immediate driver supply but also the remaining budget, thereby constraining future decisions. The goal is to learn a policy that maximizes cumulative platform utility over the day while respecting the budget cap.

A further complication is scalability. Optimizing a subsidy for each individual order--driver pair in real time is computationally prohibitive at city scale, and overly personalized incentives can raise fairness and consistency concerns among drivers facing similar conditions. As a result, industrial systems commonly deploy a two-stage mechanism: the platform first decides a city-level subsidy intensity $\lambda_t$ for each time window $t$, and then deterministically maps $\lambda_t$ to pair-level incentives (e.g., as a function of order value and pickup distance) via a Lagrangian-dual-derived rule. This design offers low-latency execution and a transparent, consistent incentive structure.

Even with this low-dimensional action space, learning an effective controller is challenging. The ride-hailing market is stochastic and non-stationary, and standard reinforcement learning (RL) approaches can be unstable and data-hungry in such settings~\cite{guo2024generative}. More importantly, the trial-and-error exploration required by online RL is misaligned with strict budget caps: aggressive exploration may cause irreversible loss, premature budget exhaustion, or operational violations. These constraints motivate offline learning, where policies are learned purely from historical logs.

In this work, we draw on recent advances in generative modeling and propose D$^3$-Subsidy (Dynamic Driver-side Diffusion-based Subsidy), a diffusion-based framework for safe and deployable city-level subsidy control. To bridge the train--inference gap in sequential decision making, we introduce a prefix-conditional diffusion scheme that learns to generate plausible future city trajectories conditioned on an observed history prefix. To explicitly promote budget feasibility, we incorporate a constraint-aware score that penalizes infeasible trajectories. We further propose a context-conditioned inverse dynamics module to decode the control signal $\lambda_t$ under specific operational contexts, and adopt multi-city pretraining with parameter-efficient fine-tuning for scalable transfer across heterogeneous cities. Our main contributions are summarized as follows:
\begin{itemize}[leftmargin=*, itemsep=0pt,topsep=2pt]
\item We formulate city-level driver subsidy allocation as a budget-constrained online sequential decision-making problem, capturing the coupling between current incentives, future feasibility, and platform-level KPIs under market stochasticity.
\item We propose D$^3$-Subsidy, a diffusion-based offline control framework featuring (i) prefix-conditional diffusion for deployment-consistent trajectory generation, (ii) a constraint-aware score for budget-feasible planning, and (iii) a context-conditioned inverse dynamics decoder for accurate and controllable action inference.
\item We evaluate our method on real-world data from three cities and validate it in production A/B tests, achieving a 1.59\% increase in \texttt{Rides} and a 2.06\% improvement in \texttt{GMV} under the same budget constraints.
\end{itemize}

\section{Preliminary}

\subsection{Operations in Ride-hailing Platforms}
\label{app:ride-hailing}
In a ride-hailing platform, operational decisions arise from the continuous interactions among passengers, drivers, and the platform. When a passenger submits a trip request, the platform first assesses the local supply-demand conditions (e.g., nearby driver availability and anticipated demand) and then determines a driver-side subsidy to influence driver participation. This subsidy is offered to nearby drivers as additional compensation for accepting and completing the order, thereby shaping drivers' acceptance decisions based on perceived profitability and travel distance \cite{chen2025grab}.
Given drivers' responses to the subsidy offer, the platform proceeds to match the order to one or more candidate drivers and updates system states accordingly (e.g., driver availability, spatial distribution of demand). This closed-loop process repeats over time and across regions, forming a large-scale, dynamic marketplace in which subsidy interventions serve as a key lever for steering driver behavior and improving driver-passenger matching outcomes.

\begin{figure}
    \centering
    \includegraphics[width=1\linewidth]{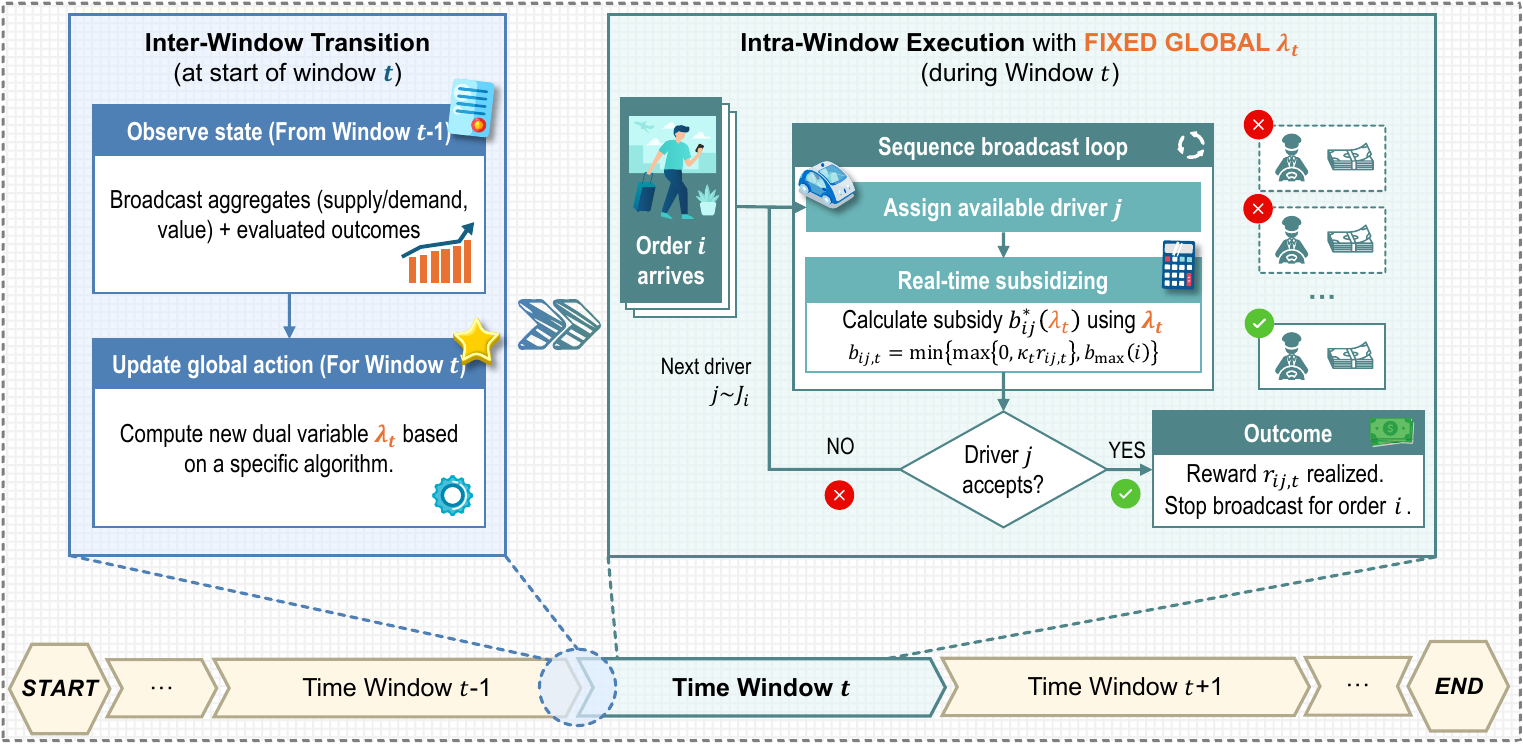}
    \caption{Problem Formulation}
    \Description{A diagram illustrating the problem formulation.}
    \label{fig:problem_formulation}
\end{figure}

\subsection{Problem Formulation} \label{sec:problem-formulation}

We consider a broadcast-based order dispatching system \cite{chen2025dynamic}. When an order $i$ is broadcast to a candidate driver $j$, the platform may offer a monetary subsidy $b_{ij}$ to incentivize acceptance and completion. Each subsidy is bounded by an order-specific cap:
\begin{equation}
0 \le b_{ij} \le b_{\max(i)}, \qquad \forall\, i,j.
\end{equation}

Let $p_{ij}(b_{ij})$ denote the probability that driver $j$ completes order $i$ under subsidy $b_{ij}$, and let $r_{ij}$ be the platform revenue obtained upon completion. To ensure sustainable operations, the platform enforces a global daily subsidy-rate constraint: total subsidy spending should not exceed a fraction $C\in(0,1)$ of total revenue. We allow a tolerance $\delta\ge 0$ to accommodate demand volatility and measurement noise. This tolerance is specified by the platform and kept fixed across all experiments. In expectation, the constraint can be written as
\begin{equation}
\frac{\sum_{i,j} b_{ij}\, p_{ij}(b_{ij})}
     {\sum_{i,j} r_{ij}\, p_{ij}(b_{ij})}
\le C + \delta.
\end{equation}

The platform's objective is to decide subsidies $\{b_{ij}\}$ to maximize expected total revenue subject to the subsidy-rate constraint and box constraints:
\begin{equation}
\label{eq:primal-general}
\begin{aligned}
\max_{b_{ij}} \quad &
\sum_{i,j} r_{ij}\, p_{ij}(b_{ij}) , \\
\text{s.t.}\quad &
\sum_{i,j} b_{ij}\, p_{ij}(b_{ij})
-
(C + \delta) \sum_{i,j} r_{ij}\, p_{ij}(b_{ij})
\le 0, \\
& 0 \le b_{ij} \le b_{\max(i)}, \qquad \forall i,j.
\end{aligned}
\end{equation}

\paragraph{Dual relaxation and closed-form solution.}
To obtain a deployable mapping with analytical transparency, we adopt a locally linear approximation of the completion probability with respect to incentives. Consistent with empirical evidence that completion likelihood increases with incentives~\cite{chen2025grab}, we use
\begin{equation}
\label{eq:linear_completion}
p_{ij}(b_{ij}) \approx a_{ij}\, b_{ij}, \qquad a_{ij} > 0,
\qquad 0 \le b_{ij}\le b_{\max(i)}.
\end{equation}
Here $a_{ij}$ is a function of observable order- and context-level characteristics (e.g., pickup distance and order value) and does not depend on driver-identity features. Therefore, any two driver--order pairs with the same observable characteristics share the same $a_{ij}$, yielding a uniform, non-discriminatory subsidy rule (i.e., no driver-specific differentiated incentives). In practice, $p_{ij}(b)$ is bounded in $[0,1]$ and can be modeled by a monotone response such as a logit specification. Since any smooth monotone function on $[0,b_{\max}]$ can be uniformly approximated by a piecewise-linear function, the dual-based mapping derived below extends to such nonlinear response models by replacing Equation~\eqref{eq:linear_completion} with its piecewise-linear approximation, leading to a piecewise closed-form policy with the same clipping structure.

Introducing a Lagrange multiplier $\lambda\ge 0$ for the subsidy-rate constraint, the Lagrangian of Equation~\eqref{eq:primal-general} becomes
\begin{equation}
\label{eq:lagrangian-linear}
\mathcal{L}(b,\lambda)
=
\sum_{i,j}
\Big[
(1 + \lambda (C+\delta))\, r_{ij} a_{ij} b_{ij}
- \lambda a_{ij} b_{ij}^2
\Big].
\end{equation}

Because $\mathcal{L}(b,\lambda)$ is separable across pairs $(i,j)$, the inner maximization admits a closed-form solution. Define $\kappa = (C + \delta + 1/\lambda)/2$. The optimal subsidy is
\begin{equation}
\label{eq:b-optimal-clipped}
b_{ij}^{*}(\lambda)
=
\min\Big\{
\max\big\{0,\ \kappa r_{ij}\big\},
\ b_{\max(i)}
\Big\}.
\end{equation}
Equation~\eqref{eq:b-optimal-clipped} shows that the pair-level subsidy is proportional to reward $r_{ij}$, scaled by a global dual variable that reflects the tightness of the daily budget.

\begin{figure*}[!tb] 
\centering 
\includegraphics[width=0.85\linewidth]{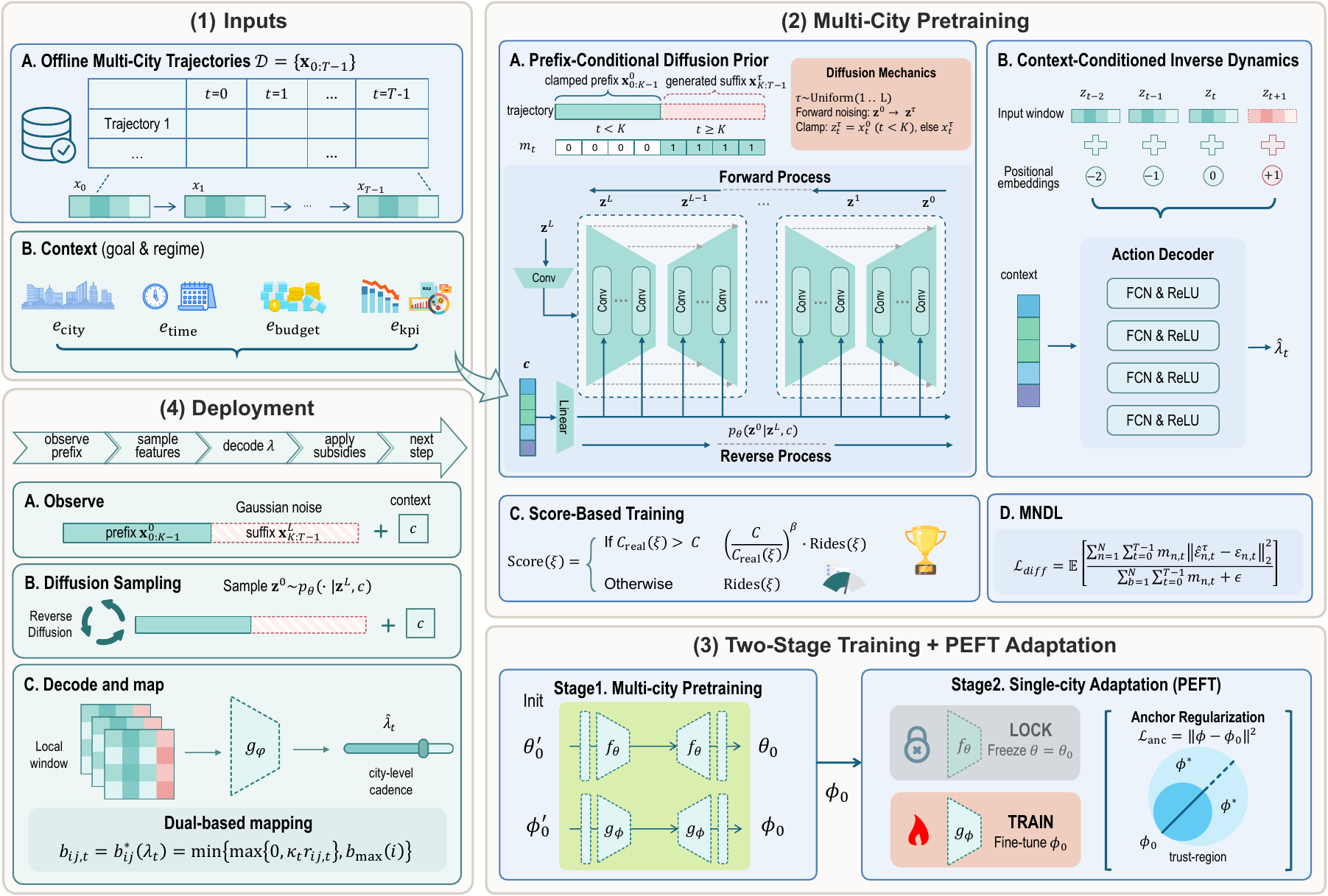} 
\caption{Overview of the proposed D$^3$-Subsidy framework.}
\label{fig:model_structure}
\end{figure*}

\subsection{Sequential Subsidy Decision-Making}
\label{sec:mdp_subsidy}

Equation~\eqref{eq:b-optimal-clipped} specifies a tractable pair-level rule given a dual variable, but directly optimizing $b_{ij}$ for millions of order--driver pairs in real time is still impractical. In production, platforms therefore update a single city-level control signal $\lambda_t$ at a coarse cadence (e.g., every 2 minutes) and apply the induced incentives within each time window. We model this online control process as a Markov Decision Process (MDP).

\paragraph{State and action.}
Let $s_t\in\mathcal{S}$ denote the observable market state at time window $t$
(e.g., supply--demand gap and spatial distribution). To track constraint
satisfaction over the day, we augment the state with the realized subsidy
rate w.r.t. \texttt{GMV}:
\begin{equation}
\rho_t
=
\frac{\sum_{t'\le t}\sum_{(i,j)\in\mathcal{E}_{t'}} b_{ij,t'}\, y_{ij,t'}}
     {\sum_{t'\le t}\sum_{(i,j)\in\mathcal{E}_{t'}} g_{ij,t'}\, y_{ij,t'}},
\end{equation}
where $\mathcal{E}_{t'}$ is the set of broadcasted order--driver pairs in period $t'$,
$y_{ij,t'}\in\{0,1\}$ indicates whether order $i$ is completed by driver $j$,
and $g_{ij,t'}$ denotes the \texttt{GMV} of pair $(i,j)$ if completed.
The augmented state is $x_t=(s_t,\rho_t)$, and the action is the scalar city-level control $\lambda_t$.

\paragraph{From city-level control to pair-level subsidies.}
Given $\lambda_t$, the platform deterministically computes incentives using Equation~\eqref{eq:b-optimal-clipped}:
\begin{equation}
\label{eq:subsidy_mapping_mdp}
b_{ij,t} = \min\Big\{ \max\big\{0,\ \kappa_t r_{ij,t}\big\}, \ b_{\max(i)} \Big\}.
\end{equation}
This mapping serves as a low-dimensional interface, converting one scalar decision into fine-grained subsidies at scale.

\paragraph{Reward and dynamics.}
The period reward $r_t$ aggregates the target KPI (e.g., \texttt{Rides} or \texttt{GMV}) over the time window. After executing $\lambda_t$ and observing outcomes $\{y_{ij,t}\}$, the realized subsidy rate $\rho_t$ is updated accordingly, and the system transitions to $x_{t+1}$ under stochastic market dynamics $\mathcal{P}(\cdot\mid x_t,\lambda_t)$. The platform seeks a policy $\pi(\lambda_t\mid x_t)$ that maximizes long-term performance while keeping $\rho_t$ within the subsidy-rate cap.

\section{D$^3$-Subsidy}
\label{sec:d3}

\subsection{Overview}
\label{sec:d3_overview}
Building on the sequential decision formulation in Section~\ref{sec:mdp_subsidy}, we propose
D$^3$-Subsidy, a diffusion-based offline decision framework for platform-controlled, city-level driver-side subsidy control.
In line with industrial practice, the platform updates a single city-level control variable $a_t=\lambda_t$ at a coarse cadence, and then induces fine-grained pair-level subsidies through the dual-based mapping in Equation~\eqref{eq:b-optimal-clipped}.

Given offline trajectories collected from multiple cities, D$^3$-Subsidy learns a transferable and controllable decision prior that can be deployed under heterogeneous operating regimes (e.g., different subsidy-rate tightness and KPI preferences).
The framework consists of two coupled modules: (i) a prefix-conditional diffusion model that generates plausible future city trajectories conditioned on an observed history prefix, and (ii) a context-conditioned inverse dynamics model that decodes the city-level control $\lambda_t$ from short state windows under the same context.
To improve feasibility under strict subsidy-rate caps, we further introduce a constraint-aware scoring objective, and adopt a two-stage training pipeline (multi-city pretraining followed by single-city adaptation) to balance generalization and local calibration.

\subsection{Prefix-Conditional Diffusion Model}
\label{sec:d3_diffusion_model}
To avoid ambiguity, we use the subscript $t\in\{0,\ldots,T-1\}$ to denote the index of the time window, and the superscript $\tau\in\{0,1,\ldots,L\}$ to denote the diffusion step. In particular, $\tau=0$ denotes a clean (undiffused) trajectory, while larger $\tau$ indicates higher noise.

We represent each city-day as a length-$T$ clean trajectory
$\xi=\mathbf{x}_{0:T-1}^{0}=(x_0^{0},\ldots,x_{T-1}^{0})$, where each state is
\begin{equation}
x_t^{0} = [\, s_t \,;\, \rho_t \,]\in \mathbb{R}^{d}.
\end{equation}
Here $s_t$ summarizes the observable market condition and $\rho_t$ is the realized subsidy rate defined in Section~\ref{sec:mdp_subsidy}. Each trajectory is paired with a context vector $\mathbf{c}$ that encodes spatial information, temporal semantics, the budget regime, and the target \texttt{Rides}.

Given an observed history prefix $\mathbf{x}^{0}_{0:K-1}$, our goal is to learn a conditional generative prior over the future suffix:
\begin{equation}
p_\theta\!\left(\mathbf{x}_{K:T-1}^{0}\mid \mathbf{x}_{0:K-1}^{0},\mathbf{c}\right).
\end{equation}
Crucially, the prefix corresponds to realized history and must remain immutable. Therefore, in both training and inference, we treat the prefix as a hard condition and diffuse/denoise only the suffix. As illustrated in Figure~\ref{fig:pc_diffusion_concept}, we clamp the realized history prefix and apply diffusion-based denoising/sampling only to the future suffix.

We define a single diffusion variable $\mathbf{z}^{\tau}=\{z_t^{\tau}\}_{t=0}^{T-1}$ over the entire horizon:
\begin{equation}
\begin{aligned}
z_t^{\tau}
&\triangleq
\begin{cases}
x_t^{0}, & t<K,\\
x_t^{\tau}, & t\ge K,
\end{cases}\\
\mathbf{z}^{\tau}
&\triangleq
(\mathbf{x}_{0:K-1}^{0},\ \mathbf{x}_{K:T-1}^{\tau}).
\end{aligned}
\end{equation}
where $\mathbf{x}_{K:T-1}^{\tau}$ denotes the noised suffix at diffusion step $\tau$. Thus, the prefix is enforced for all $\tau$, while only the suffix evolves along the diffusion chain.
We write $\alpha_\tau=1-\beta_\tau$ and $\bar{\alpha}_\tau=\prod_{i=1}^{\tau}\alpha_i$.

\subsubsection{Forward process} We adopt the DDPM parameterization~\cite{ho2020ddpm} with a cosine noise schedule~\cite{nichol2021improved}, but restrict noising to the suffix.
Formally, the forward distribution over the entire $\mathbf{z}^{\tau}$ factorizes as
\begin{equation}
\label{eq:forward_z_factor}
q(\mathbf{z}^{\tau}\mid \mathbf{x}^{0},K)
=
\prod_{t=0}^{K-1}\delta\!\left(z_t^{\tau}-x_t^{0}\right)
\;\cdot\;
\prod_{t=K}^{T-1}\mathcal{N}\!\left(
z_t^{\tau};\ \sqrt{\bar{\alpha}_\tau}\,x_t^{0},\ (1-\bar{\alpha}_\tau)I_d
\right),
\end{equation}
where $\delta(\cdot)$ denotes the Dirac delta distribution, which places all
its mass at $z_t^{\tau}=x_t^{0}$ and thus leaves the prefix tokens
$t\in[0,K\!-\!1]$ unchanged (i.e., no noise is added to the prefix), and $I_d$
is the $d\times d$ identity matrix.
Equivalently, for each $t\ge K$ we can sample
\begin{equation}
\label{eq:forward_suffix_sample}
z_{t}^{\tau}
=
\sqrt{\bar{\alpha}_\tau}\,x_{t}^{0}
+
\sqrt{1-\bar{\alpha}_\tau}\,\varepsilon_t,
\qquad
\varepsilon_t\sim\mathcal{N}(0,I_d),
\end{equation}
and set $z_t^{\tau}=x_t^{0}$ for all $t<K$ deterministically.

\subsubsection{Reverse process} At inference time, we initialize the suffix as
$\{z_t^{L}\}_{t=K}^{T-1}\sim\mathcal{N}(0,I_d)$ and clamp the prefix by setting
$z_t^{L}=x_t^{0}$ for all $t<K$.
We then sample $\mathbf{z}^{\tau-1}$ from $\mathbf{z}^{\tau}$ for $\tau=L,\ldots,1$.

Let $f_\theta$ be a Temporal U-Net denoiser that takes $(\mathbf{z}^{\tau},\tau,\mathbf{c})$
and outputs a sequence of noise predictions
$\hat{\boldsymbol{\varepsilon}}^{\tau}=f_\theta(\mathbf{z}^{\tau},\tau,\mathbf{c})\in\mathbb{R}^{T\times d}$.
We denote by $\hat{\varepsilon}^{\tau}_{t}=[\hat{\boldsymbol{\varepsilon}}^{\tau}]_t$
the $t$-th slice (time window) of this output and only use it on suffix positions:
\begin{equation}
\hat{\varepsilon}^{\tau}_{t}=\big[f_\theta(\mathbf{z}^{\tau},\tau,\mathbf{c})\big]_t,\qquad t\ge K.
\end{equation}

For $t\ge K$, we form the implied clean estimate
\begin{equation}
\label{eq:z0_pred}
\hat{x}_{t}^{0}
=
\frac{1}{\sqrt{\bar{\alpha}_\tau}}
\left(
z_{t}^{\tau}-\sqrt{1-\bar{\alpha}_\tau}\,\hat{\varepsilon}_{t}^{\tau}
\right),
\end{equation}
and use the standard DDPM reverse mean/variance:
\begin{equation}
\label{eq:reverse_mean_var_z}
\mu_{\theta,t}^{\tau}
=
\frac{1}{\sqrt{\alpha_\tau}}
\left(
z_{t}^{\tau}
-
\frac{\beta_\tau}{\sqrt{1-\bar{\alpha}_\tau}}\,
\hat{\varepsilon}_{t}^{\tau}
\right),
\qquad
\tilde{\beta}_\tau
=
\beta_\tau\cdot \frac{1-\bar{\alpha}_{\tau-1}}{1-\bar{\alpha}_{\tau}}.
\end{equation}

Crucially, the reverse transition is defined over the entire $\mathbf{z}$ with hard prefix constraints:
\begin{equation}
\label{eq:reverse_z_factor}
p_\theta(\mathbf{z}^{\tau-1}\mid \mathbf{z}^{\tau},\tau,\mathbf{c},K)
=
\prod_{t=0}^{K-1}\delta\!\left(z_t^{\tau-1}-x_t^{0}\right)
\;\cdot\;
\prod_{t=K}^{T-1}\mathcal{N}\!\left(
z_t^{\tau-1};\ \mu_{\theta,t}^{\tau},\ \tilde{\beta}_\tau I_d
\right).
\end{equation}
The final generated future is given by the suffix of $\mathbf{z}^{0}$, i.e.,
$\mathbf{x}_{K:T-1}^{0}=\mathbf{z}_{K:T-1}^{0}$.
For completeness, pseudocode for the prefix-clamped forward noising and reverse sampling procedures is provided in Algorithms \ref{alg:pc_forward_noising} and \ref{alg:pc_traj_diffusion_sampling}.

\begin{algorithm}
\caption{Forward Noising for Prefix-Conditional Diffusion}
\label{alg:pc_forward_noising}
\begin{algorithmic}[1]
\Require Clean trajectory $\mathbf{x}_{0:T-1}^{0}$, prefix length $K$, diffusion step $\tau\in\{1,\ldots,L\}$, $\bar{\alpha}_\tau$, state dimension $d$
\Ensure Noised variable $\mathbf{z}^{\tau}=(\mathbf{x}_{0:K-1}^{0},\mathbf{x}_{K:T-1}^{\tau})$
\State $\mathbf{z}^{\tau}_{0:K-1} \gets \mathbf{x}^{0}_{0:K-1}$ \Comment{clamp prefix (no noise)}
\State $\boldsymbol{\varepsilon}_{K:T-1}\sim \mathcal{N}(\mathbf{0},I_d)$ \Comment{$(T\!-\!K)\times d$ suffix noise}
\State $\mathbf{z}^{\tau}_{K:T-1} \gets \sqrt{\bar{\alpha}_\tau}\,\mathbf{x}^{0}_{K:T-1}+\sqrt{1-\bar{\alpha}_\tau}\,\boldsymbol{\varepsilon}_{K:T-1}$
\State \Return $\mathbf{z}^{\tau}$
\end{algorithmic}
\end{algorithm}

\begin{algorithm}
\caption{Prefix-Conditional Trajectory Diffusion Sampling}
\label{alg:pc_traj_diffusion_sampling}
\begin{algorithmic}[1]
\Require Observed prefix $\mathbf{x}_{0:K-1}^{0}$, context $\mathbf{c}$,
        diffusion steps $L$,
        schedule $\{\alpha_\tau,\beta_\tau,\bar{\alpha}_\tau,\tilde{\beta}_\tau\}_{\tau=1}^{L}$,
        denoiser $f_\theta$
\Ensure Sampled future suffix $\mathbf{x}_{K:T-1}^{0}$

\Statex \textbf{Initialize (clamp prefix, randomize suffix).}
\For{$t \gets 0$ \textbf{to} $K-1$}
    \State $z_t^{L} \gets x_t^{0}$ \Comment{hard prefix}
\EndFor
\For{$t \gets K$ \textbf{to} $T-1$}
    \State $z_t^{L} \sim \mathcal{N}(0,I_d)$ \Comment{noisy suffix}
\EndFor

\Statex \textbf{Reverse diffusion (denoise suffix with hard prefix).}
\For{$\tau \gets L, L-1, \ldots, 1$}
    \State $\mathbf{z}^{\tau} \gets \mathrm{Clamp}_K(\mathbf{z}^{\tau};\mathbf{x}_{0:K-1}^{0})$
    \State $\hat{\boldsymbol{\varepsilon}}^{\tau} \gets f_\theta(\mathbf{z}^{\tau},\tau,\mathbf{c})$ \Comment{$\hat{\boldsymbol{\varepsilon}}^{\tau}\in\mathbb{R}^{T\times d}$}

    \For{$t \gets K$ \textbf{to} $T-1$}
        \State $\hat{\varepsilon}^{\tau}_{t} \gets [\hat{\boldsymbol{\varepsilon}}^{\tau}]_{t}$ \Comment{$t$-th slice (time window)}
        \State $\mu_{\theta,t}^{\tau} \gets \frac{1}{\sqrt{\alpha_\tau}}\Big(
            z_{t}^{\tau}
            - \frac{\beta_\tau}{\sqrt{1-\bar{\alpha}_\tau}}\,\hat{\varepsilon}_{t}^{\tau}
        \Big)$
        \If{$\tau>1$}
            \State $z_t^{\tau-1} \sim \mathcal{N}(\mu_{\theta,t}^{\tau},\tilde{\beta}_\tau I_d)$
        \Else
            \State $z_t^{0} \gets \mu_{\theta,t}^{1}$ \Comment{final step: no noise}
        \EndIf
    \EndFor

    \State $\mathbf{z}^{\tau-1} \gets \mathrm{Clamp}_K(\mathbf{z}^{\tau-1};\mathbf{x}_{0:K-1}^{0})$ \Comment{maintain hard prefix}
\EndFor

\State \Return $\mathbf{x}_{K:T-1}^{0} \gets \mathbf{z}_{K:T-1}^{0}$
\end{algorithmic}
\end{algorithm}

\begin{figure}[!tb]
    \centering
    \includegraphics[width=\linewidth]{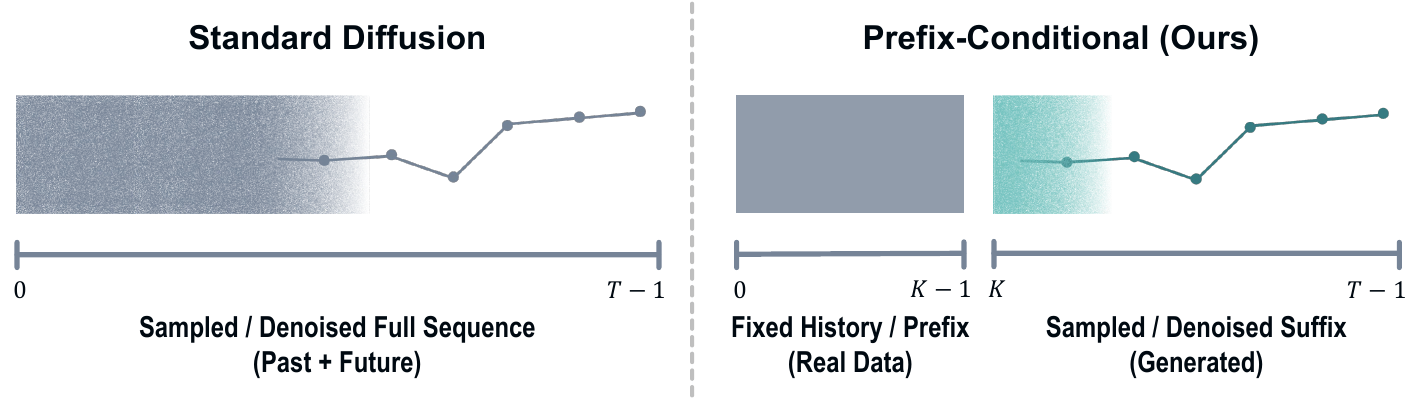}
    \caption{Comparison of standard trajectory diffusion and prefix-conditional diffusion.}
    \label{fig:pc_diffusion_concept}
\end{figure}

\subsubsection{Mask-Normalized Denoising Loss (MNDL)} A practical challenge in prefix-conditioned training is that the number of supervised
positions varies across examples and minibatches due to (i) different prefix lengths $K$
and (ii) padding. Without accounting for this variation, the denoising loss scale becomes
coupled to the effective suffix length, causing large loss/gradient fluctuations and
destabilizing optimization. To address this, we propose the
MNDL, which normalizes the denoising error by the number of valid (unmasked) suffix positions.

For a minibatch of $N$ trajectories $\{\xi^{n}\}_{n=1}^{N}$, let $m_{n,t}\in\{0,1\}$ be a binary mask indicating whether timestep $t$ in trajectory $\xi^{n}$ is a valid suffix position (i.e., $t\ge K$ and not padded). We set $m_{n,t}=0$ for the fixed prefix $t<K$ and for any padded timesteps.
Noise variables $\varepsilon_{n,t}$ are sampled only for suffix positions ($t\ge K$); the deterministic prefix contributes neither noise nor gradients.
We optimize
\begin{equation}
\label{eq:d3_mask_norm_denoise}
\mathcal{L}_{\mathrm{diff}}(\theta)
=
\mathbb{E}\left[
\frac{
\sum_{n=1}^{N}\sum_{t=0}^{T-1} m_{n,t}\,
\left\|\big[f_\theta(\mathbf{z}_{n}^{\tau},\tau,\mathbf{c}_{n})\big]_t-\varepsilon_{n,t}\right\|_2^2
}{
\sum_{n=1}^{N}\sum_{t=0}^{T-1} m_{n,t}
}
\right].
\end{equation}

MNDL decouples the loss magnitude from the effective horizon, substantially reducing minibatch-to-minibatch variance when $K$ varies or when trajectories contain padding.

\subsection{Context-Conditioned Inverse Dynamics}
\label{sec:d3_inv_dyn}
Rather than directly diffusing actions, we infer the city-level control via inverse dynamics. This design is motivated by city-scale subsidy control, where the same state transition can correspond to different actions under different operating targets and constraint regimes.
For example, a transition that is desirable under a growth-oriented target may be infeasible (or suboptimal) under a tighter subsidy cap. As a result, an unconditional inverse dynamics model $g(z_{t-2:t+1}^{0})$ tends to average across heterogeneous objectives and produces actions that are poorly calibrated for deployment-time steering.

We therefore propose a context-conditioned inverse dynamics decoder that treats the operating target and constraint regime as an explicit task specification.
Let $\mathbf{c}$ encode city-level embedding, temporal covariates (e.g., time-of-day and day-of-week), budget regime, and target \texttt{Rides} preference. We decode the city-level control as
\begin{equation}
\hat{\lambda}_t
=
g_\phi(z_{t-2}^{0},z_{t-1}^{0},z_t^{0},z_{t+1}^{0},\mathbf{c}).
\end{equation}
Here $z_t^{0}$ denotes the clean (final) prefix-clamped trajectory at diffusion step $\tau=0$: during training, $z_t^{0}$ is constructed from the ground-truth trajectory, whereas at inference it is obtained from the sampled trajectory $\mathbf{z}^{0}$ produced by the diffusion prior.

This conditioning makes the decoder controllable: varying $\mathbf{c}$ at deployment enables the same diffusion prior to yield different valid controls consistent with the desired objective--constraint trade-off.
We train $g_\phi$ with a mask-normalized mean-squared error
\begin{equation}
\mathcal{L}_{\mathrm{inv}}(\phi)
=
\mathbb{E}\left[
\frac{
\sum_{n,t} m^{\mathrm{inv}}_{n,t}\,
\|\hat{\lambda}_{n,t}-\lambda_{n,t}\|_2^2
}{
\sum_{n,t} m^{\mathrm{inv}}_{n,t}+\epsilon
}
\right],
\end{equation}
where $m^{\mathrm{inv}}_{n,t}$ masks padded timesteps in the $n$-th trajectory.

Finally, the decoded city-level control $\hat{\lambda}_t$ is converted into deployable pair-level subsidies by applying the closed-form dual mapping in Equation~\eqref{eq:b-optimal-clipped} (with $r_{ij}$ replaced by $r_{ij,t}$ for window $t$), i.e., $b_{ij,t}=b_{ij}^*(\hat{\lambda}_t)$.

\subsection{Constraint-Aware Score}
\label{sec:d3_awr}
To balance the primary KPI with strict subsidy-rate feasibility, we introduce a constraint-aware score defined on a trajectory $\xi$.
Let $\texttt{Rides}(\xi)$ be the number of \texttt{Rides} and let $C_{\mathrm{real}}(\xi)$ be the realized subsidy rate over $\xi$. Given a subsidy-rate cap $C\in(0,1)$ and a penalty exponent $\beta>0$, we define
\begin{equation}
\label{eq:score_def}
\mathrm{Score}(\xi)
=
\begin{cases}
\left(\dfrac{C}{C_{\mathrm{real}}(\xi)}\right)^{\beta}\, \texttt{Rides}(\xi),
& \text{if } C_{\mathrm{real}}(\xi) > C,\\[2mm]
\texttt{Rides}(\xi),
& \text{otherwise.}
\end{cases}
\end{equation}
When the subsidy-rate constraint is satisfied, the score equals the KPI. When the cap is violated, the score applies a smooth multiplicative penalty that increases with the degree of violation. The exponent $\beta$ controls the penalty strength and provides a simple knob to trade off \texttt{Rides} improvement against subsidy-rate compliance.

\subsection{Two-Stage Training}
\label{sec:d3_two_stage}
We first pretrain $(f_\theta,g_\phi)$ on pooled multi-city trajectories and condition both modules on context $\mathbf{c}$. This stage learns (i) a transferable diffusion prior over city dynamics and (ii) a generic inverse-dynamics decoder for the city-level control $a_t=\lambda_t$.

To improve deployment calibration on a target city $c^*$ using offline-only data, we apply parameter-efficient fine-tuning (PEFT): we freeze the diffusion model parameters and fine-tune only the inverse dynamics decoder. Let $(\theta_0,\phi_0)$ denote pretrained parameters. We optimize
\begin{equation}
\label{eq:ft_obj}
\min_{\phi}\quad
\mathcal{L}_{\mathrm{inv}}(\phi)
+
\lambda_{\text{anc}}\|\phi-\phi_0\|_2^2,
\end{equation}
where the anchor term acts as a trust-region surrogate to prevent excessive drift from the pretrained decoder. In practice, we use a smaller learning rate than in pretraining and update only $\phi$ while keeping $\theta=\theta_0$ fixed.

\section{Experiments}
\label{sec:exp}

\definecolor{OursGreen}{HTML}{ECF4F3}

\subsection{Experimental Setup}
\subsubsection{Dataset}
We compile an offline dataset from Didi broadcast logs to evaluate city-level subsidy control. It contains three aligned versions that use the same cities/days, state--action setup, and evaluation protocol, and differ only in the aggregation window (2/5/10 minutes). The data are stored as city-day trajectories. At each step, we aggregate city-level supply--demand signals and context, log the average online control $\lambda$, and compute key KPIs including \texttt{Rides}, \texttt{GMV}, and driver revenue (\texttt{DRV}). The dataset covers 133 Brazilian cities over 28 days; full statistics are in Table~\ref{tab:data_stats}.

\begin{table}[htbp]
  \caption{Data statistics under different time windows.}
  \label{tab:data_stats}
  \centering
  \small
  \begin{tabular}{cccc}
    \toprule
    \textbf{Parameters} & \textbf{2-min} & \textbf{5-min} & \textbf{10-min} \\
    \midrule
    Trajectories (city$\times$day) & 3,724 & 3,724 & 3,724 \\
    Cities / Days & 133 / 28 & 133 / 28 & 133 / 28 \\
    Trajectory length & 720 & 288 & 144 \\
    State dimension & 20 & 20 & 20 \\
    Action dimension & 1 & 1 & 1 \\
    Return-To-Go dimension & 1 & 1 & 1 \\
    Action range & (0, 30] & (0, 30] & (0, 30] \\
    \bottomrule
  \end{tabular}
\end{table}

\subsubsection{Baselines} We compare against the following baselines:
\begin{itemize}[leftmargin=*, itemsep=0pt]
    \item Online: Didi's production subsidy policy (predict-then-optimize).
    \item BC: behavior cloning---supervised policy learning from logged actions.
    \item BCQ \cite{fujimoto2019off}: offline RL that restricts actions to the dataset support (and nearby) to curb out-of-distribution (OOD) overestimation.
    \item CQL \cite{kumar2020conservative}: conservative offline RL that penalizes the $Q$-function to reduce overestimation.
    \item IQL \cite{kostrikov2022offline}: offline RL that improves BC via advantage-weighted regression without explicit OOD action maximization.
    \item TD3+BC \cite{fujimoto2021minimalist}: TD3 updates with a BC regularizer to limit deviation from logged behavior.
    \item DT (Decision Transformer) \cite{chen2021decision}: an autoregressive Transformer policy that generates actions from trajectory history conditioned on a target return-to-go.
    \item DD (Decision Diffuser) \cite{ajayconditional}: a diffusion-based generative policy that produces actions via iterative denoising, conditioned on the current context and a target return.
\end{itemize}

All offline RL baselines are trained to maximize the $\mathrm{Score}(\xi)$.
Therefore, the compared offline RL methods already correspond to soft-constrained policy learning under the subsidy-rate cap.

\subsubsection{Evaluation} 
For offline evaluation, we use the unified split described above. The main test set consists of 7 consecutive days from 3 cities (City A, City B, and City C), yielding 21 city-day trajectories. The cold-start test set consists of all city-day trajectories from 3 additional held-out cities (City D, City E, and City F). All remaining trajectories are used for training. We report $\mathrm{Score}(\xi)$ as the primary metric and also report \texttt{Rides}, \texttt{GMV}, and \texttt{DRV}. All policies are evaluated via closed-loop rollouts in Didi's high-fidelity production simulator. Simulator fidelity is verified by replaying the production Online policy, achieving 7.49\% MAPE on daily \texttt{Rides} over the held-out trajectories.

At each time window, the policy maps the current city-level state to an action $\lambda_t$. The simulator applies $\lambda_t$ together with the broadcast context within the time window, and uses its internal predictive model to generate the resulting KPIs and the next state. Rolling out the full day yields a simulated city-day trajectory $\xi$, from which we compute $\mathrm{Score}(\xi)$ via Equation~\eqref{eq:score_def}. 
For each city, we report the mean $\mathrm{Score}(\xi)$ over its 7 test days for each method and compare algorithms using these city-level averages.
In production, we allow a small upper tolerance band and count a violation only if $C_{\text{real}}(\xi) > C+\delta$. We additionally report under-utilization as $\texttt{UnderGap}(\xi)=\max(0,C-C_{\text{real}}(\xi))$, where smaller is better.

\begin{table*}[t]
  \caption{Offline evaluation score. Scores are averaged over the 7 hold-out days. Best in \textbf{bold}, second-best \underline{underlined}.}
  \label{tab:score}
  \centering
  \small
  \begin{threeparttable}
  \begin{tabular}{c c | >{\columncolor{OursGreen}}c cccccccc}
    \toprule
    City & Time Window & D$^3$-Subsidy (Ours) & Online & BC & BCQ & CQL & IQL & TD3+BC & DT & DD \\
    \midrule
    \multirow{4}{*}{City A} 
      & 2-min   & \textbf{10802.72} & 10368.08 & 10441.40 & 10494.38 & 10382.75 & 10685.53 & 10638.74 & \underline{10784.55} & 10730.26 \\
      & 5-min   & \textbf{10824.43} & 10359.22 & 10324.27 & 10445.51 & 10513.69 & 10612.73 & 10452.29 & \underline{10618.92} & 10488.95 \\
      & 10-min  & \textbf{10712.05} & 10341.50 & 10423.46 & 10260.77 & 10659.68 & 10262.76 & \underline{10696.46} & 10658.45 & 10332.50 \\
      & Average & \textbf{10779.73} & 10356.27 & 10396.38 & 10400.22 & 10518.71 & 10520.34 & 10595.83 & \underline{10687.31} & 10517.24 \\
    \midrule

    \multirow{4}{*}{City B} 
      & 2-min   & \textbf{2008.16} & 1940.09 & 1910.12 & 1992.22 & 1981.83 & 1963.18 & 1981.78 & 1957.06 & \underline{1992.83} \\
      & 5-min   & \textbf{2000.00} & 1938.57 & 1911.20 & 1883.80 & 1958.27 & 1953.19 & 1958.20 & 1954.61 & \underline{1987.89} \\
      & 10-min  & \textbf{2018.06} & 1936.16 & 1938.49 & 1908.35 & 1967.94 & 1973.27 & \underline{1988.49} & 1984.76 & 1983.93 \\
      & Average & \textbf{2008.74} & 1938.27 & 1919.94 & 1928.12 & 1969.35 & 1963.21 & 1976.16 & 1965.48 & \underline{1988.22} \\
    \midrule

    \multirow{4}{*}{City C} 
      & 2-min   & \textbf{1736.80} & 1650.70 & 1628.70 & 1694.30 & 1669.96 & 1656.00 & 1686.09 & \underline{1722.55} & 1709.86 \\
      & 5-min   & \textbf{1792.48} & 1643.23 & 1613.38 & 1657.01 & 1661.04 & 1642.01 & 1650.17 & 1674.31 & \underline{1737.39} \\
      & 10-min  & \textbf{1718.31} & 1631.43 & 1670.16 & 1597.11 & 1689.27 & 1679.28 & 1649.69 & 1672.04 & \underline{1717.02} \\
      & Average & \textbf{1749.20} & 1641.79 & 1637.41 & 1649.47 & 1673.42 & 1659.10 & 1661.98 & 1689.63 & \underline{1721.42} \\
    \midrule

    \multicolumn{2}{c|}{Overall Average} 
      & \textbf{4845.89} & 4645.44 & 4651.24 & 4659.27 & 4720.49 & 4714.22 & 4744.66 & \underline{4780.81} & 4742.29 \\
    \bottomrule
  \end{tabular}

  \begin{tablenotes}[flushleft]
    \footnotesize
    \item Note: ``Average'' rows for each city represent the mean value over time windows; ``Overall Average'' is computed by averaging over all city samples.
  \end{tablenotes}
  \end{threeparttable}
\end{table*}

\subsection{Main Results}

Table~\ref{tab:score} summarizes the city-level results on the test split. 
D$^3$-Subsidy achieves the best performance across all cities and settings, delivering the highest overall average score and consistently outperforming the online strategy.
Compared with representative offline RL baselines, our gains remain consistent. In particular, DT is the strongest competitor among baselines on average, yet still falls behind D$^3$-Subsidy. The gain holds across all windows, suggesting robust city-level subsidy control under the same subsidy-rate constraint. Since the performance trends are consistent across temporal granularities, we use the 5-minute variant in the remaining experiments for ablations and in-depth analysis.

\subsection{Ablation Studies}
We elucidate the contribution of key design components in D$^3$-Subsidy through an ablation study by evaluating the following variants:
\begin{itemize}[leftmargin=*, itemsep=0pt]
    \item D$^3$-Subsidy-C: removes the Condition in the inverse dynamics module, predicts the control signal without trajectory-conditioned decoding.
    \item D$^3$-Subsidy-M: removes multi-city pretraining, training only on the target city's trajectories.
    \item D$^3$-Subsidy-P: removes the PEFT/fine-tuning stage and directly applies the pretrained backbone for inference.
\end{itemize}

The results are reported in Table~\ref{tab:ablation}. As shown, removing any key component leads to a clear degradation in score, validating that the performance gain of D$^3$-Subsidy is not driven by a single component. Removing conditioning (C), multi-city pretraining (M), and PEFT (P) reduces the score by 3.1\%, 4.9\%, and 4.3\%, respectively, with the largest drop from w/o multi-city pretraining.

\begin{table}[htbp]
  \caption{Ablation results in City C (5-min).}
  \label{tab:ablation}
  \centering
  \small
  \begin{tabular}{lcc}
    \toprule
    Model & Score & \textit{Compare} \\
    \midrule
    D$^3$-Subsidy   & \textbf{1792.48} & -- \\
    D$^3$-Subsidy-C & 1737.58 & $-3.1\%$ \\
    D$^3$-Subsidy-M & 1704.91 & $-4.9\%$ \\
    D$^3$-Subsidy-P & 1715.78 & $-4.3\%$ \\
    \bottomrule
  \end{tabular}
\end{table}

\subsection{In-depth Analysis}
\subsubsection{Statistical Comparison} 

To verify that the $\mathrm{Score}(\xi)$ gains are consistent rather than driven by a few trajectories, we perform paired significance tests against each baseline on the held-out set (3 cities $\times$ 7 days, $N{=}21$).
For each trajectory $\xi$, we compute $\Delta(\xi)=\mathrm{Score}_{\text{ours}}(\xi)-\mathrm{Score}_{\text{baseline}}(\xi)$ and perform a one-sided paired $t$-test with $H_0:\mathbb{E}[\Delta]\le 0$ and $H_1:\mathbb{E}[\Delta]>0$.
Table~\ref{tab:paired_ttest} reports the mean gain, 95\% confidence interval, and $t$-statistic.
D$^3$-Subsidy yields significant improvements over all baselines (all $p<0.05$), indicating that the advantage is robust and reproducible.

\begin{table}[htbp]
  \caption{Statistical significance of D$^3$-Subsidy improvements.}
  \label{tab:paired_ttest}
  \centering
  \small
  \begin{threeparttable}
  \begin{tabular}{lcccc}
    \toprule
    Baseline & Mean Diff. & 95\% Confidence Interval & $t_{20}$ & $p$-value \\
    \midrule
    Online   & $+225.28$ & $[120.79,\,329.78]$ & 4.50 & $<\!0.001$ \\
    BC       & $+256.01$ & $[155.05,\,356.96]$ & 5.29 & $<\!0.001$ \\
    BCQ      & $+210.18$ & $[108.35,\,312.02]$ & 4.31 & $<\!0.001$ \\
    CQL      & $+161.29$ & $[78.98,\,243.60]$  & 4.09 & $<\!0.001$ \\
    IQL      & $+136.31$ & $[69.55,\,203.08]$  & 4.26 & $<\!0.001$ \\
    TD3+BC   & $+185.40$ & $[102.08,\,268.73]$ & 4.64 & $<\!0.001$ \\
    DT       & $+123.01$ & $[53.81,\,192.21]$  & 3.71 & $<\!0.001$ \\
    DD       & $+134.21$ & $[17.30,\,251.12]$  & 2.39 & 0.013 \\
    \bottomrule
  \end{tabular}
  \end{threeparttable}
\end{table}

\subsubsection{Temporal Dynamics of KPI under Subsidy Control}
Beyond average $\mathrm{Score}(\xi)$, we further inspect how different algorithms shape intra-day operational dynamics under the same subsidy rate control. Using City C as a representative case, we visualize the time evolution of key KPIs at the 5-minute time window. 

As shown in Figure~\ref{fig:cumulative}, our method delivers steady daily improvements over the online strategy in \texttt{Rides}, \texttt{GMV}, and \texttt{DRV}, suggesting the overall gain does not rely on trade-offs among these indicators.
Notably, the advantage accumulates steadily over time within each day, suggesting a stable and sustained improvement rather than a chance lead.
To understand where the cumulative gain comes from, Figure~\ref{fig:per_slot} reports per-window KPI trajectories, showing that our method preserves a similar daily pattern to the online strategy while remaining consistently higher across most time windows.
Figure~\ref{fig:tr} visualizes the daily subsidy rate as deviation from the cap $C$ (reported in this form for confidentiality), where the shaded area indicates the allowed upper tolerance band.
Compared with the online strategy, D$^3$-Subsidy ends near $C$ (or slightly above) with no overshoot beyond $C+\delta$ on the held-out test days, while maintaining a consistently small $\texttt{UnderGap}$. Results for the other held-out test cities are provided in Appendix~\ref{sec:kpi_fig} (Figures~\ref{fig:cumulative_041}--\ref{fig:tr_213}).

\begin{figure}[!tb]
    \centering
    \begin{subfigure}[t]{0.49\linewidth}
        \centering
        \includegraphics[width=\linewidth]{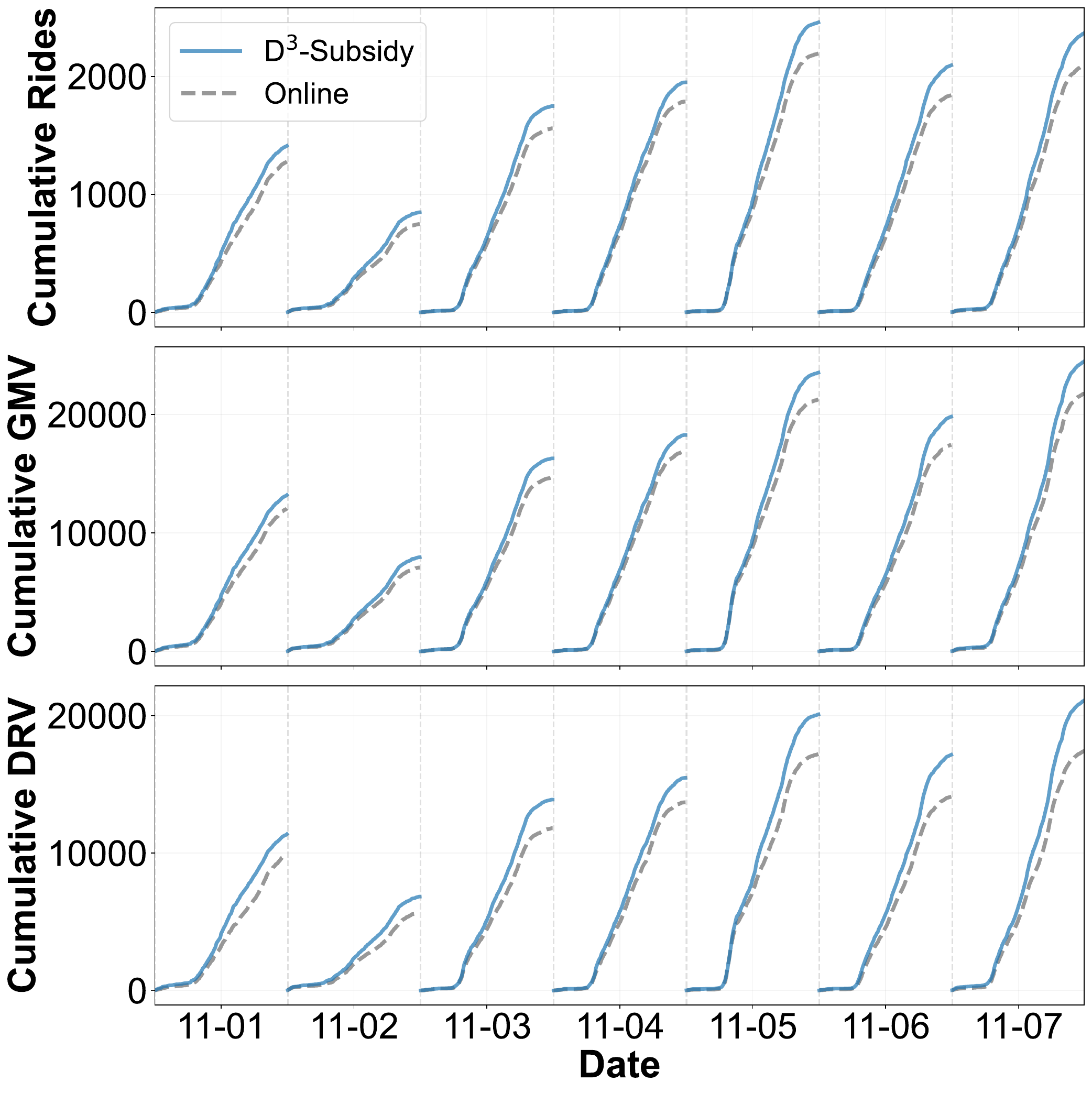}
        \caption{Cumulative}
        \label{fig:cumulative}
    \end{subfigure}\hfill
    \begin{subfigure}[t]{0.49\linewidth}
        \centering
        \includegraphics[width=\linewidth]{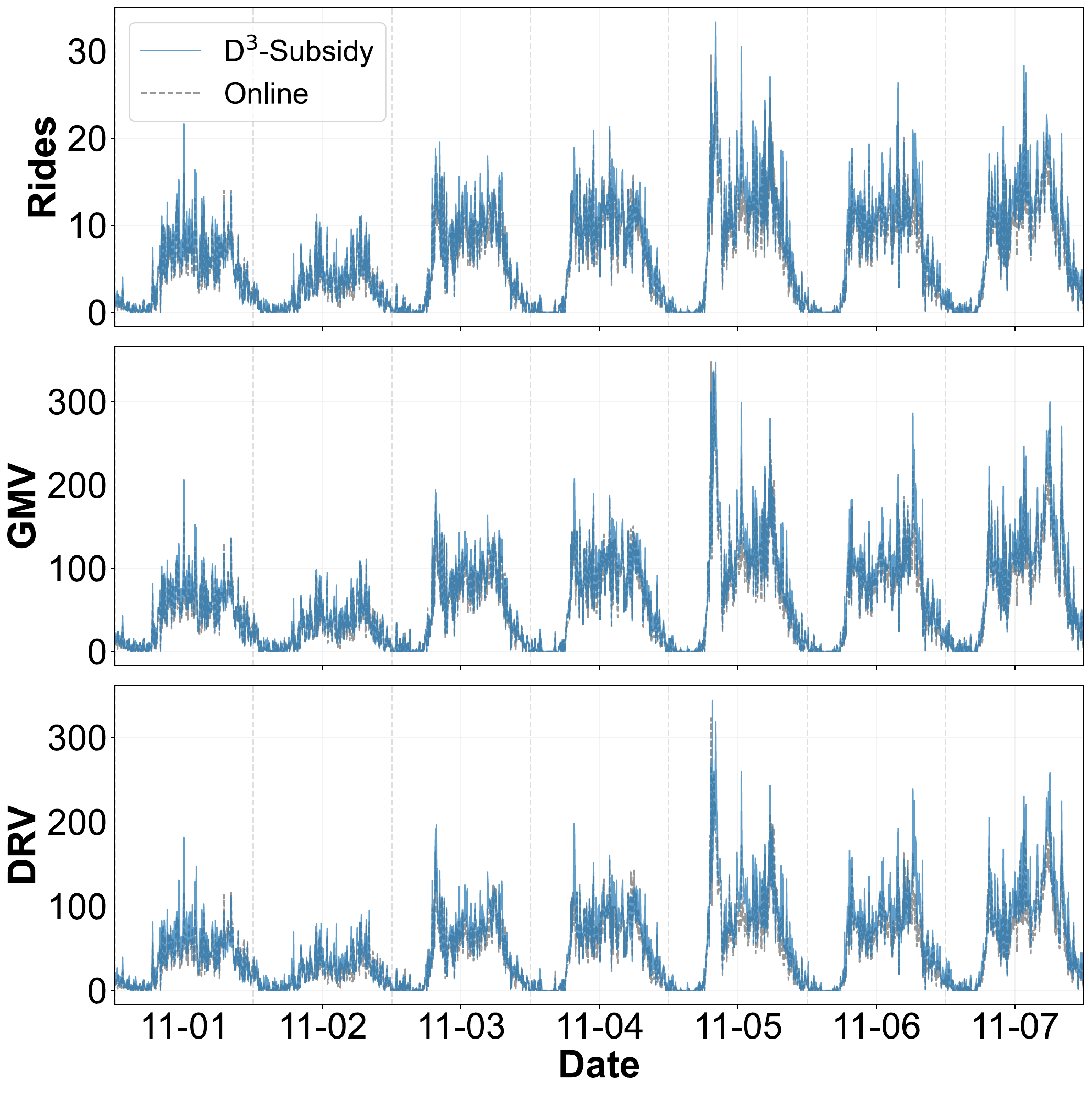}
        \caption{Per-window}
        \label{fig:per_slot}
    \end{subfigure}
    \caption{Intra-day \texttt{Rides}, \texttt{GMV} and \texttt{DRV} dynamics in City C.}
    \label{fig:kpi_dynamics_city_c}
\end{figure}

\begin{figure}[!tb]
    \centering
    \includegraphics[width=0.8\linewidth]{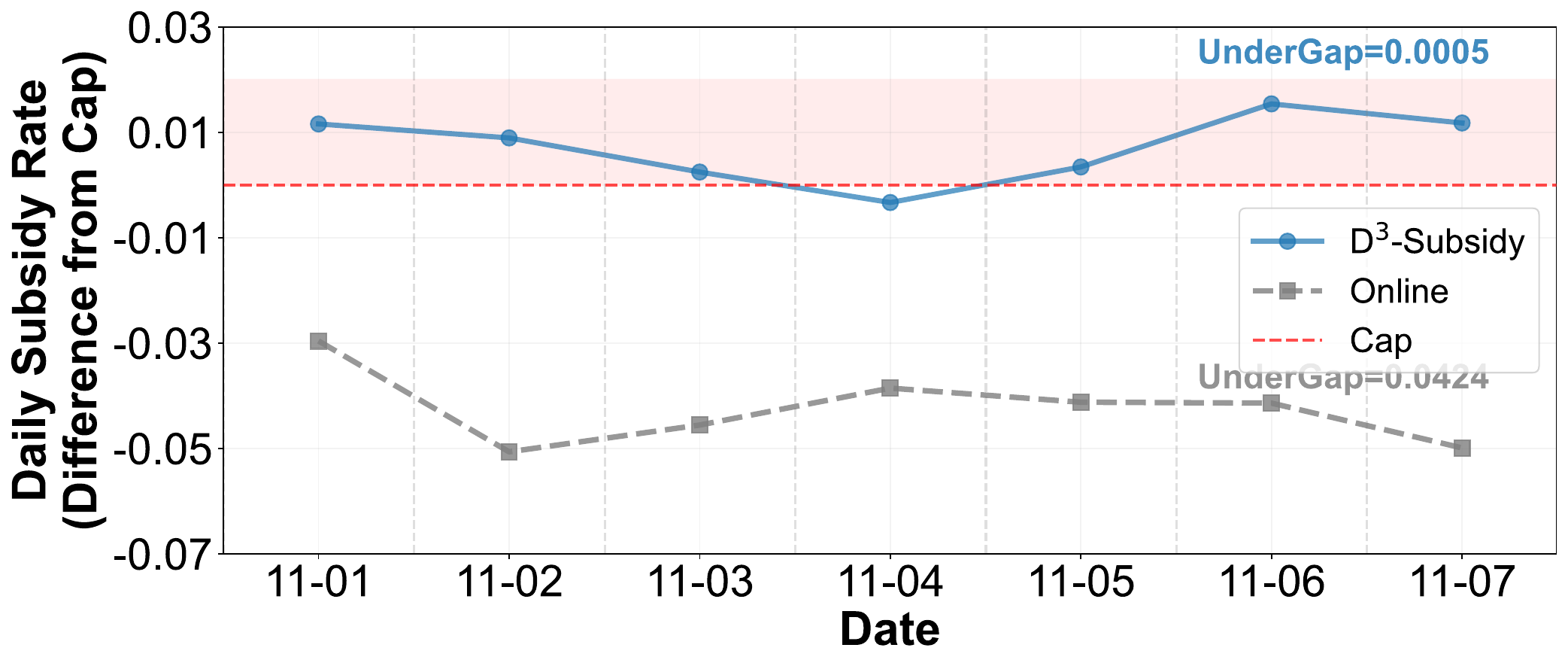}
    \caption{Daily Subsidy Rate in City C.}
    \label{fig:tr}
\end{figure}

\subsubsection{Target-KPI Controllability Analysis}
\label{sec:kpi_steering}

A practical requirement for subsidy control is deployment-time steerability: the platform should be able to shift operating targets by adjusting the target-KPI inputs without retraining.
To test this, we scale the target-KPI signals in the context by $\gamma \in \{0.2,0.4,\ldots,2.0\}$ while keeping the subsidy-rate constraint fixed, and roll out the resulting policies aggregated over the three held-out test cities.
As shown in Figure~\ref{fig:kpi_conditional_policy_steering}, varying $\gamma$ induces the expected shifts in realized Score, \texttt{Rides}, and \texttt{GMV}, indicating that the model converts target-KPI conditioning into controllable subsidy decisions and outcomes.

\begin{figure}[!tb]
    \centering
    \includegraphics[width=0.6\linewidth]{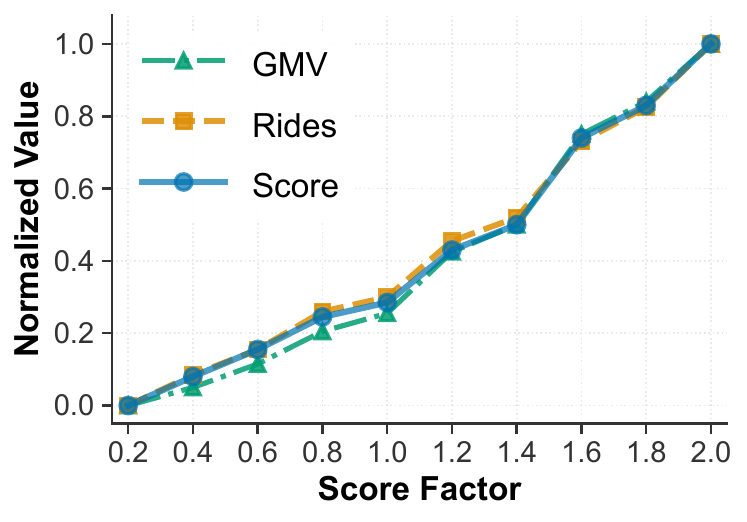}
    \caption{KPI-conditional policy steering.}
    \label{fig:kpi_conditional_policy_steering}
\end{figure}

\subsubsection{Stabilizing Diffusion Training via MNDL}

As shown in Figure \ref{fig:loss_compare_two_panels}, replacing the vanilla denoising loss with MNDL yields more stable optimization, and the loss curve accordingly becomes smoother with large fluctuations significantly reduced. MNDL normalizes the denoising error by the number of valid (unmasked) suffix steps, preventing the loss scale from drifting when the effective prediction horizon changes due to padding or varying prefix lengths. We also note that the diffusion loss may appear numerically larger after introducing MNDL. This increase is expected because the normalization denominator changes, so the absolute magnitudes across the two losses are not on the same scale. In practice, we focus on the improved stability and consistent convergence behavior rather than the raw loss values.

\begin{figure}[!tb]
  \centering
  \begin{subfigure}[t]{0.5\linewidth}
    \centering
    \includegraphics[width=\linewidth]{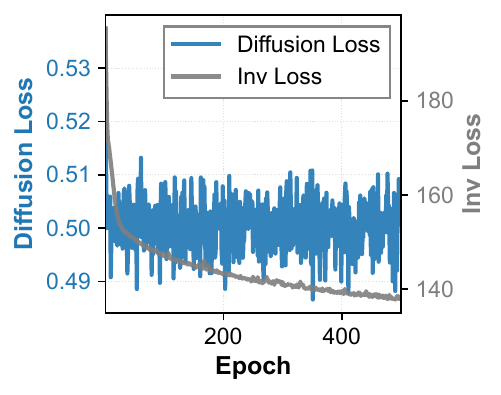}
    \caption{w/o MNDL}
    \label{fig:loss_wo_mndl}
  \end{subfigure}\hfill
  \begin{subfigure}[t]{0.5\linewidth}
    \centering
    \includegraphics[width=\linewidth]{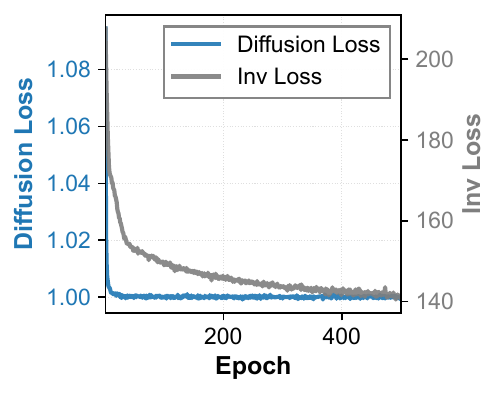}
    \caption{w/ MNDL}
    \label{fig:loss_w_mndl}
  \end{subfigure}

  \caption{Training loss comparison under different settings.}
  \label{fig:loss_compare_two_panels}
\end{figure}

\subsubsection{Cold-Start Performance}
\label{cold-start}

Cold-start deployment is common in practice because a new city may have no city-specific offline data at launch. We evaluate cold-start transfer on the three additional held-out cities whose all city-days are excluded from training. PEFT is not applicable in this setting since it requires target-city offline data to fine-tune the inverse-dynamics decoder.
Table~\ref{tab:cold_start} reports results on the 5-min setting, averaged over 7 consecutive days. D$^3$-Subsidy achieves the highest average score in the cold-start setting, outperforming both the online strategy and DT, which indicates stronger cross-city generalization.

\definecolor{OursGreen}{HTML}{ECF4F3}
\begin{table}[htbp]
  \caption{Offline evaluation score on cold-start cities. Best in \textbf{bold}, second-best \underline{underlined}.}
  \label{tab:cold_start}
  \centering
  \small
  \begin{threeparttable}
  \begin{tabular}{c | >{\columncolor{OursGreen}}c cc}
    \toprule
    City & D$^3$-Subsidy (Ours) & Online & DT \\
    \midrule
    City D & \textbf{532.20} & \underline{532.06} & 531.74 \\
    City E & \textbf{285.54} & \underline{277.44} & 266.66 \\
    City F & \textbf{1152.65} & \underline{1142.31} & 1132.58 \\
    \midrule
    Average & \textbf{656.80} & \underline{650.60} & 643.66 \\
    \bottomrule
  \end{tabular}
  \end{threeparttable}
\end{table}

\subsection{Sensitivity Analysis}
\subsubsection{Sensitivity to Reverse Diffusion Steps}
We sweep reverse diffusion steps $\tau \in \{10, 50, 100, 150\}$ in City C while keeping other settings fixed (Figure~\ref{fig:diff_steps}). Performance peaks at 50 steps: fewer steps degrade quality and more steps do not consistently help.

Overall, this sweep confirms that our default choice of 50 steps is both effective and reasonably robust to step variations.

\begin{figure}[htbp]
    \centering
    \includegraphics[width=0.6\linewidth]{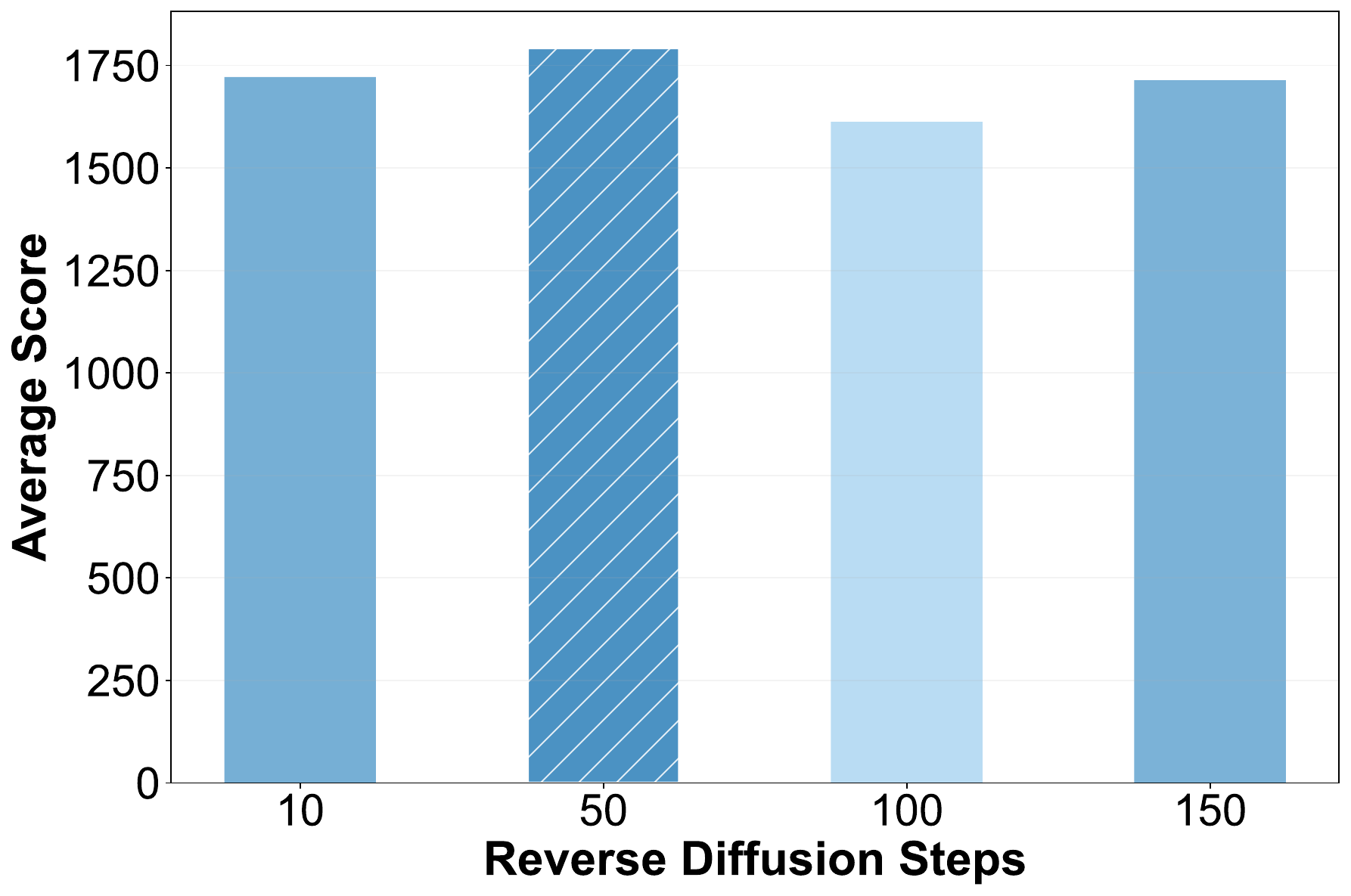}
    \caption{Score under different diffusion steps in City C.}
    \label{fig:diff_steps}
\end{figure}

\subsubsection{Sensitivity to the Penalty Exponent $\beta$}
Penalty exponent $\beta$ controls how strongly trajectories are penalized when the realized daily subsidy rate exceeds the target cap $C$. Table~\ref{tab:beta_sensitivity} reports results under $\beta\in\{0.5,1.0,2.0\}$ in City C. 
Notably, the online strategy stays below the cap in this city, so its score remains unchanged as $\beta$ varies. In contrast, D$^3$-Subsidy aims to closely track the target budget and may finish slightly above the soft cap on some days; therefore, increasing $\beta$ leads to a gradual decrease in score as expected. Importantly, D$^3$-Subsidy consistently outperforms the online strategy across all tested $\beta$ values, demonstrating that our improvement is robust and not driven by a particular choice of $\beta$.

\begin{table}[htbp]
    \centering
    \small
    \caption{Score sensitivity to penalty exponent $\beta$ in City C.}
    \label{tab:beta_sensitivity}
    \begin{tabular}{c|cc}
        \toprule
        Penalty exponent $\beta$ & D$^3$-Subsidy (Ours) & Online \\
        \midrule
        0.5 & 1792.48 & 1643.23 \\
        1.0 & 1747.28 & 1643.23 \\
        2.0 & 1662.15 & 1643.23 \\
        \bottomrule
    \end{tabular}
\end{table}

\subsection{Online A/B Test}
We further validate D$^3$-Subsidy via a 7-day online A/B test on Didi's production system (Feb 1--7, 2026) in a city, using the online strategy as the baseline. Orders are randomly assigned to the online strategy and D$^3$-Subsidy, with 50\% of orders allocated to each variant (202,279 broadcasts in total), and no other concurrent strategies or operational campaigns were running in the test scope. Both variants share the same subsidy-rate cap, traffic scope, and serving constraints. Across all 7 days, the subsidy-rate cap was never violated. Table~\ref{tab:ab_test} reports the relative lift on \texttt{Rides}, \texttt{GMV}, and \texttt{DRV}: D$^3$-Subsidy delivers consistent gains, confirming that offline improvements translate to real-market benefits while meeting the cap. Inference is operationally feasible: compared with the baseline, it adds around 20~ms per decision.

\begin{table}[htbp]
\centering
\caption{Online A/B Test Result.}
\label{tab:ab_test}
    \begin{tabular}{l|cccc}
    \toprule
    Metrics & \texttt{Rides}  & \texttt{GMV}  & \texttt{DRV} & Inference Time\\
    \midrule
    \textit{compare} & +1.59\% & +2.06\% & +2.31\%  & +20~ms \\
    \bottomrule
    \end{tabular}
\end{table}

\section{Related Work}

\noindent \textbf{Offline RL.} 
Offline RL is widely adopted in real-world systems where online exploration is prohibitively expensive or risky \cite{feng2024multi}. It learns policies from static logged data but suffers from distribution shift when generalizing beyond observed actions. Classical offline RL methods, such as BCQ \cite{fujimoto2019off} and IQL \cite{kostrikov2022offline}, mitigate this issue through behavior constraints or conservative value estimation. 
More recently, generative decision models have emerged as an alternative paradigm. \citet{chen2021decision} formulate offline RL as conditional sequence modeling using Transformers, enabling long-horizon decision making without explicit value functions. Diffusion models have also been introduced into offline RL; for example, \citet{hansen2023idql} combine diffusion-based generation with implicit Q-learning to improve robustness. 

\noindent \textbf{Diffusion Models for Decision Making.} 
Diffusion models provide strong generative priors for complex, multi-modal data and support conditional generation with diverse samples \cite{austin2021structured,chen2024overview}, making them attractive for decision modeling from offline data. Prior work has connected diffusion models to offline RL and policy learning \cite{wang2022diffusion,hansen2023idql}, as well as to behavior and trajectory modeling with temporal conditioning \cite{chenoffline,hu2023instructed}. Beyond RL, diffusion-based models have been applied to constrained decision-making and industrial planning, including autonomous driving and large-scale bidding systems \cite{zhengdiffusion,tan2025flow,guo2024generative,li2025generative}. 
In contrast, diffusion-based frameworks have not been explored for ride-hailing automated driver-subsidy control, where decisions must be made under non-stationary dynamics and subsidy-rate constraints.

\noindent \textbf{Automated Driver-Subsidy Control in Ride-Hailing.} 
Existing automated driver-subsidy control methods in ride-hailing primarily differ in subsidy granularity and decision scope. \citet{zhang2023short} study acquisition subsidies and show that they are most effective when one side of the two-sided market remains thin. \citet{zhu2021mean} formulate spatiotemporal subsidies as sequential control using a mean-field MDP and advocate zone-based surge subsidies over uniform payments. Moving to finer granularity, \citet{xie2023understanding} propose personalized incentives by estimating heterogeneous treatment effects. Yang et al.~\cite{yang2025decision} consider city-level budget allocation and propose a decision-focused learning framework that optimizes predictive models for subsidy allocation objectives. However, their formulation treats budget allocation as a static or single-stage decision problem and does not explicitly model the sequential MDP structure induced by budget consumption and evolving system dynamics.

\section{Conclusion}
In this paper, we propose D$^3$-Subsidy for automated driver-subsidy control decisions in the ride-hailing market, which can optimize the KPI and satisfy the subsidy-rate constraint at the same time. Extensive experiments on offline simulation and the  online A/B testing demonstrate the effectiveness of D$^3$-Subsidy. In the future, scholars can develop new methods to improve the efficiency and controllability of D$^3$-Subsidy. 
\newpage

\bibliographystyle{ACM-Reference-Format}
\bibliography{sample-base-arxiv}


\appendix

\section{Closed-Form Dual-Based Subsidy under a Linear Completion Model}
\label{app:dual-derivation}

We now present the complete derivation of the closed-form optimal subsidy under
the linear completion model introduced in the main text. The result is stated
as a lemma followed by its proof.

\begin{lemma}[Closed-Form Optimal Subsidy under Linear Completion]
\label{lemma:linear}
Suppose the completion probability is linear in the subsidy:
\[
p_{ij}(b_{ij}) = a_{ij} b_{ij}, \qquad a_{ij} > 0.
\]
Let \(C \in (0,1)\) be the global subsidy-rate cap. Consider the primal problem
\[
\begin{aligned}
\max_{b_{ij}} \quad &
\sum_{i,j} r_{ij} a_{ij} b_{ij},\\
\textnormal{s.t.}\quad &
\sum_{i,j} a_{ij} b_{ij}^{2}
-
(C+\delta) \sum_{i,j} r_{ij} a_{ij} b_{ij}
\le 0,\\
& 0 \le b_{ij} \le b_{\max(i)},\ \forall i,j.
\end{aligned}
\]
Let \(\lambda \ge 0\) be the Lagrange multiplier associated with the
subsidy-rate constraint. Then the optimal subsidy for each \((i,j)\)
under dual parameter \(\lambda\) (with \(\lambda>0\)) is
\[
b_{ij}^{*}(\lambda)
=
\min\big\{
\max\{0,\ \kappa r_{ij}\},
\ b_{\max(i)}
\big\},
\]
where
\[
\kappa
=
\frac{C + \delta + 1/\lambda}{2}.
\]
Letting \(\lambda^{*}\) denote the optimal dual multiplier, the final
dual-based optimal subsidy is
\[
b_{ij}^{*}
=
\min\big\{
\max\{0,\ \kappa r_{ij}\},
\ b_{\max(i)}
\big\}.
\]
\end{lemma}

\begin{proof}
Under the linear completion model, the expected reward and total subsidy are
\[
R(b) = \sum_{i,j} r_{ij} a_{ij} b_{ij},
\qquad
C(b) = \sum_{i,j} a_{ij} b_{ij}^{2}.
\]
The subsidy-rate constraint \(C(b)/R(b)\le C\) is equivalent to
\[
\sum_{i,j} a_{ij} b_{ij}^{2}
-
(C + \delta) \sum_{i,j} r_{ij} a_{ij} b_{ij}
\le 0.
\]

Introduce a Lagrange multiplier \(\lambda \ge 0\). Using the standard form
\(\mathcal{L}(b,\lambda)=R(b)-\lambda\Big(C(b)-(C+\delta)R(b)\Big)\), the Lagrangian is
\[
\mathcal{L}(b,\lambda)
=
\sum_{i,j}
\big[
(1 + \lambda (C + \delta))\, r_{ij} a_{ij} b_{ij}
- \lambda a_{ij} b_{ij}^{2}
\big].
\]
Because the Lagrangian is separable across \((i,j)\), we solve the inner
maximization independently for each pair:
\[
\max_{0 \le b_{ij} \le b_{\max(i)}}
\big[
(1 + \lambda (C + \delta))\, r_{ij} a_{ij} b_{ij}
- \lambda a_{ij} b_{ij}^{2}
\big].
\]

Ignoring the box constraint temporarily, the first-order condition is
\[
\frac{\partial \mathcal{L}}{\partial b_{ij}}
=
(1 + \lambda (C + \delta))\, r_{ij} a_{ij}
- 2 \lambda a_{ij} b_{ij}
= 0.
\]
Solving for \(b_{ij}\) yields the stationary point
\[
b_{ij}^{\circ}(\lambda)
=
\frac{(1 + \lambda (C + \delta))}{2\lambda}\, r_{ij}
=
\frac{C + \delta + 1/\lambda}{2}\, r_{ij},
\]
and defining
\[
\kappa = \frac{C + \delta + 1/\lambda}{2}
\]
gives
\[
b_{ij}^{\circ}(\lambda)=\kappa r_{ij}.
\]

The feasible region is the interval \([0, b_{\max(i)}]\). Thus the optimal
solution is the projection of the stationary point onto this interval:
\[
b_{ij}^{*}(\lambda)
=
\min\big\{
\max\{0,\ \kappa r_{ij}\},
\ b_{\max(i)}
\big\}.
\]

The dual function is
\[
g(\lambda)
=
\sum_{i,j}
\max_{0 \le b_{ij} \le b_{\max(i)}}
\big[
(1 + \lambda (C + \delta))\, r_{ij} a_{ij} b_{ij}
- \lambda a_{ij} b_{ij}^{2}
\big],
\]
and the dual problem is \(\min_{\lambda \ge 0} g(\lambda)\).
Letting \(\lambda^{*}\) denote the optimal multiplier and substituting it into
the expression above yields
\[
b_{ij}^{*}
=
\min\big\{
\max\{0,\ \kappa r_{ij}\},
\ b_{\max(i)}
\big\},
\]
which completes the proof.
\end{proof}

\section{General Completion Models: Optimal Subsidy Characterization}
\label{app:lemma-general}

We now provide a general structural result for the dual-based subsidy when the
completion function $p_{ij}(b)$ is nonlinear (e.g., logistic or saturating).
The result shows that the optimal subsidy retains the same clipped form as in
the linear case, except that the interior optimum is determined implicitly by
a one-dimensional equation. This lemma serves as the theoretical foundation for
the general expression in the main text.

\begin{lemma}[Structure of Dual-Based Optimal Subsidy for General Completion Models]
\label{lemma:general-p}
Consider the primal subsidy optimization problem in the main text and assume
that the completion function $p_{ij}(b)$ satisfies:
\begin{itemize}
    \item[\textnormal{(A1)}] $p_{ij}$ is continuously differentiable on 
    $[0, b_{\max(i)}]$,
    \item[\textnormal{(A2)}] $p_{ij}'(b) \ge 0$ (monotonicity),
    \item[\textnormal{(A3)}] $p_{ij}''(b) \le 0$ (concavity),
    \item[\textnormal{(A4)}] $p_{ij}(0)=0$ and $p_{ij}(b_{\max(i)}) \le 1$.
\end{itemize}
Let $\lambda \ge 0$ be the Lagrange multiplier associated with the
subsidy-rate constraint, and define, for each $(i,j)$,
\[
F_{ij}(b;\lambda)
\triangleq
\big[(1 + \lambda (C + \delta))\, r_{ij} - \lambda b\big]\, p_{ij}'(b)
- \lambda\, p_{ij}(b).
\]
Then:

\begin{enumerate}
    \item There exists at most one $b_{ij}^{\circ}(\lambda) \in (0, b_{\max(i)})$
    satisfying
    \[
    F_{ij}(b_{ij}^{\circ}(\lambda);\lambda)=0.
    \]

    \item The dual-based optimal subsidy for each pair $(i,j)$ is given by
    \begin{equation}
    \label{eq:lemma-clipped}
    b_{ij}^{*}(\lambda)
    =
    \min\big\{
    \max\{0,\ \widehat{b}_{ij}(\lambda)\},
    \ b_{\max(i)}
    \big\},
    \end{equation}
    where $\widehat{b}_{ij}(\lambda)=b_{ij}^{\circ}(\lambda)$ if the root lies in
    $(0,b_{\max(i)})$, and any value outside this interval otherwise.
\end{enumerate}
\end{lemma}

\begin{proof}
Fix $(i,j)$ and $\lambda \ge 0$. Using the standard Lagrangian form
\(\mathcal{L}=f-\lambda g\) with
\(f=r_{ij}p_{ij}(b)\) and \(g=b\,p_{ij}(b)-(C+\delta)\,r_{ij}p_{ij}(b)\),
the inner maximization problem decouples across $(i,j)$ and reduces to
\[
\max_{0 \le b \le b_{\max(i)}}\;
\Big( \big[(1 + \lambda (C+\delta))\, r_{ij} - \lambda b\big]\, p_{ij}(b) \Big).
\]

For an interior maximizer $b \in (0,b_{\max(i)})$, the first-order condition is
\[
\frac{d\mathcal{L}}{db}
=
\big[(1 + \lambda (C+\delta))\, r_{ij} - \lambda b\big]\, p_{ij}'(b)
- \lambda\, p_{ij}(b)
= 0,
\]
which is exactly $F_{ij}(b;\lambda)=0$.

Under (A1)--(A3), $p_{ij}$ is increasing and concave. For common completion
curves (e.g., logistic or saturating), the equation $F_{ij}(b;\lambda)=0$ admits
at most one root in $(0,b_{\max(i)})$, establishing statement (1).

To characterize the maximizer, observe that the sign of
$\frac{d\mathcal{L}}{db}$ at the endpoints determines whether an endpoint is
optimal; when neither endpoint is optimal, the (unique) root gives the interior
optimum. Therefore the optimal solution is obtained by projecting
$b_{ij}^{\circ}(\lambda)$ (if it exists) onto the feasible interval
$[0,b_{\max(i)}]$, which yields exactly the clipped form in
\eqref{eq:lemma-clipped}. This establishes statement (2) and completes the
proof.
\end{proof}

\section{Temporal Dynamics of KPIs in Additional Cities}
\label{sec:kpi_fig}

This appendix provides the temporal KPI trajectories for the two additional held-out test cities (City A and City B), complementing the City C case in the main text. For each city, we show (i) cumulative \texttt{Rides}/\texttt{GMV}/\texttt{DRV}, (ii) per-window \texttt{Rides}/\texttt{GMV}/\texttt{DRV} curves (5-minute), and (iii) the day-to-date subsidy rate as deviation from the cap $C$.

Figures~\ref{fig:cumulative_041}--\ref{fig:tr_041} correspond to City A and exhibit the same pattern as in City C: D$^3$-Subsidy delivers consistent gains across KPIs while keeping the realized subsidy rate tightly aligned with the cap under the same tolerance-band criterion.
Figures~\ref{fig:cumulative_213}--\ref{fig:tr_213} correspond to City B and confirm similar behavior in another held-out market, with steady KPI uplift and stable subsidy rate control throughout the day.

\begin{figure}[htbp]
    \centering
    \includegraphics[width=0.9\linewidth]{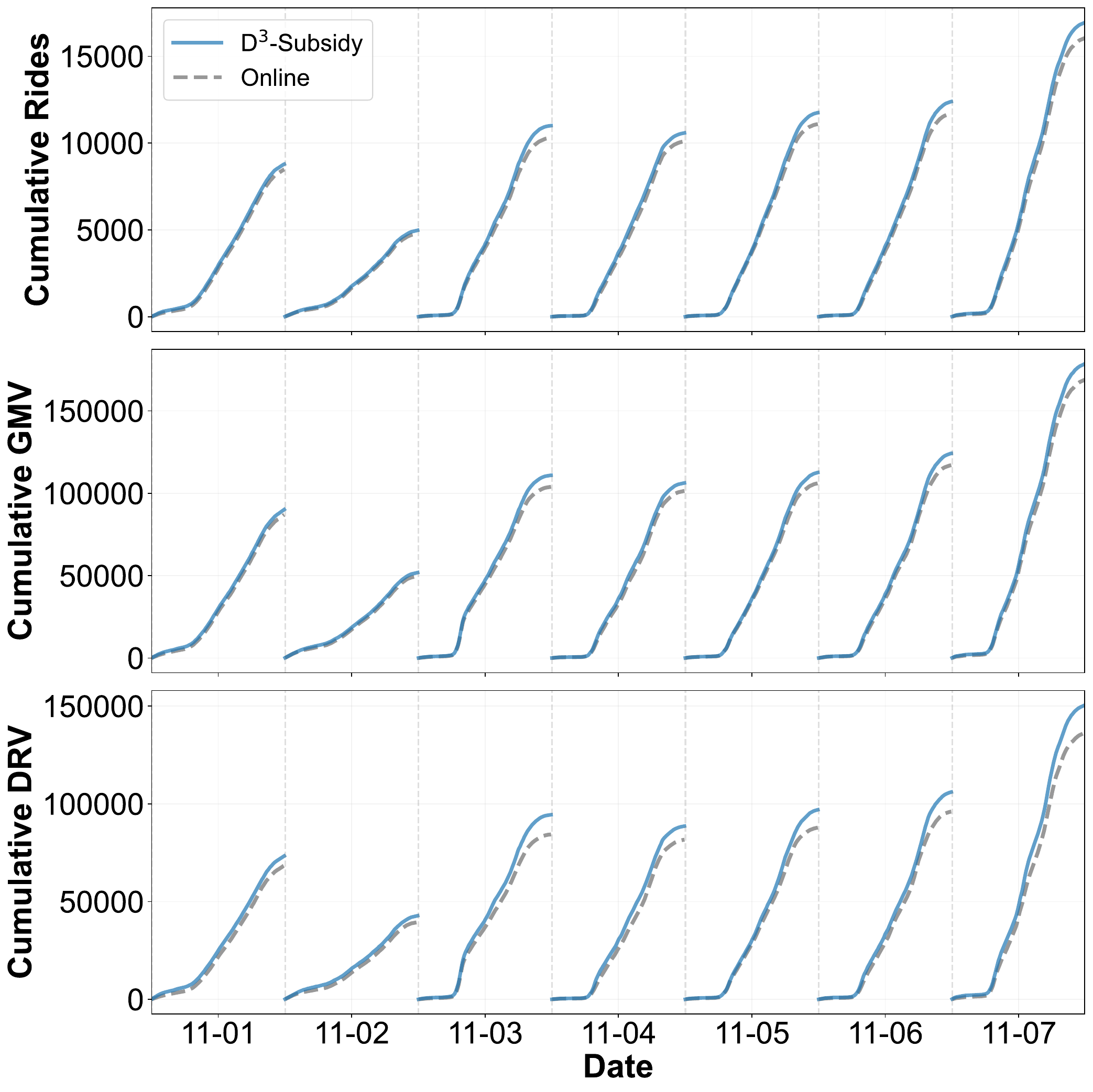}
    \caption{Cumulative \texttt{Rides}, \texttt{GMV} and \texttt{DRV} in City A.}
    \label{fig:cumulative_041}
\end{figure}

\begin{figure}[htbp]
    \centering
    \includegraphics[width=0.9\linewidth]{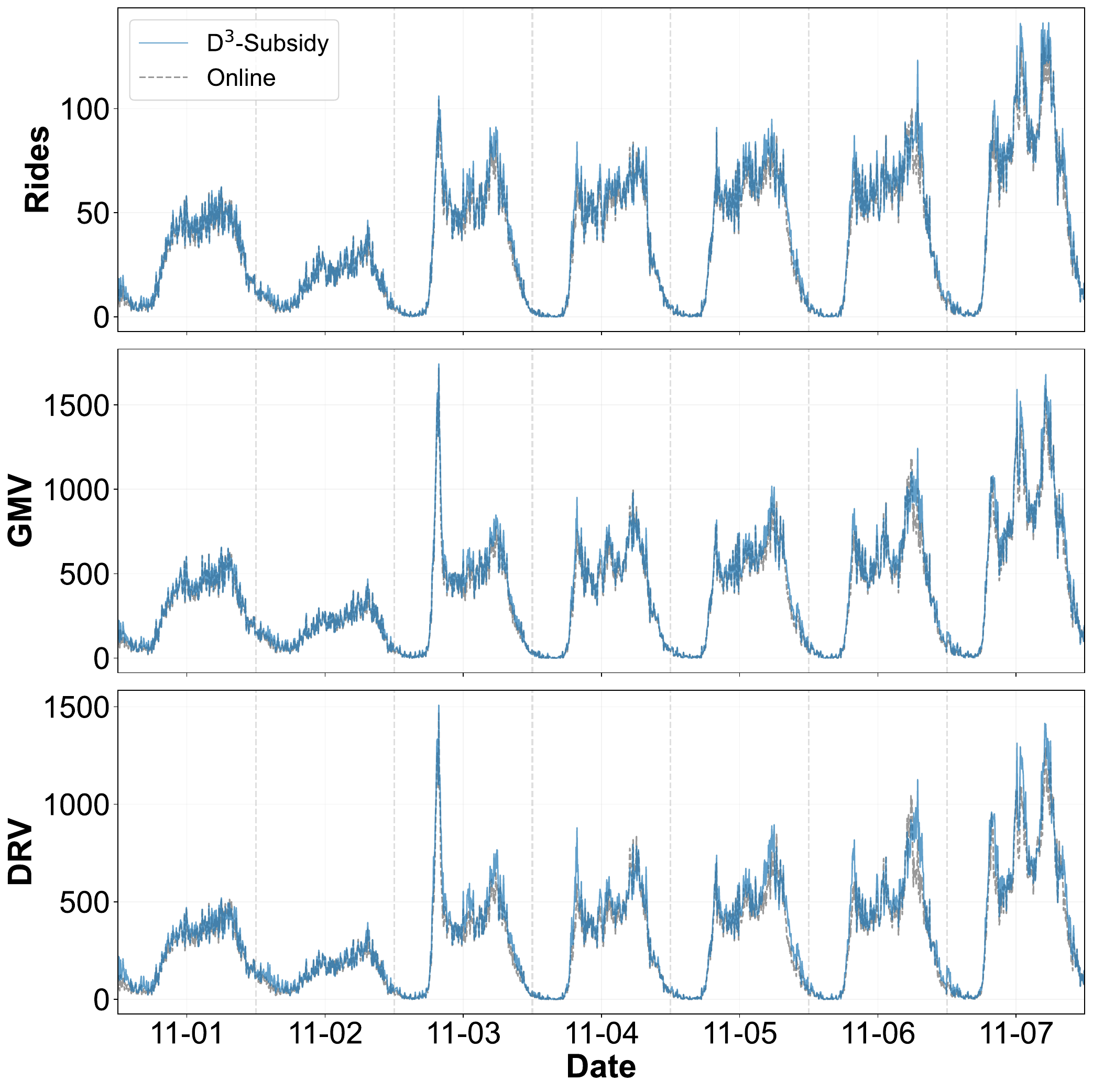}
    \caption{Per-Window \texttt{Rides}, \texttt{GMV} and \texttt{DRV} in City A.}
    \label{fig:per_slot_041}
\end{figure}

\begin{figure}[htbp]
    \centering
    \includegraphics[width=0.92\linewidth]{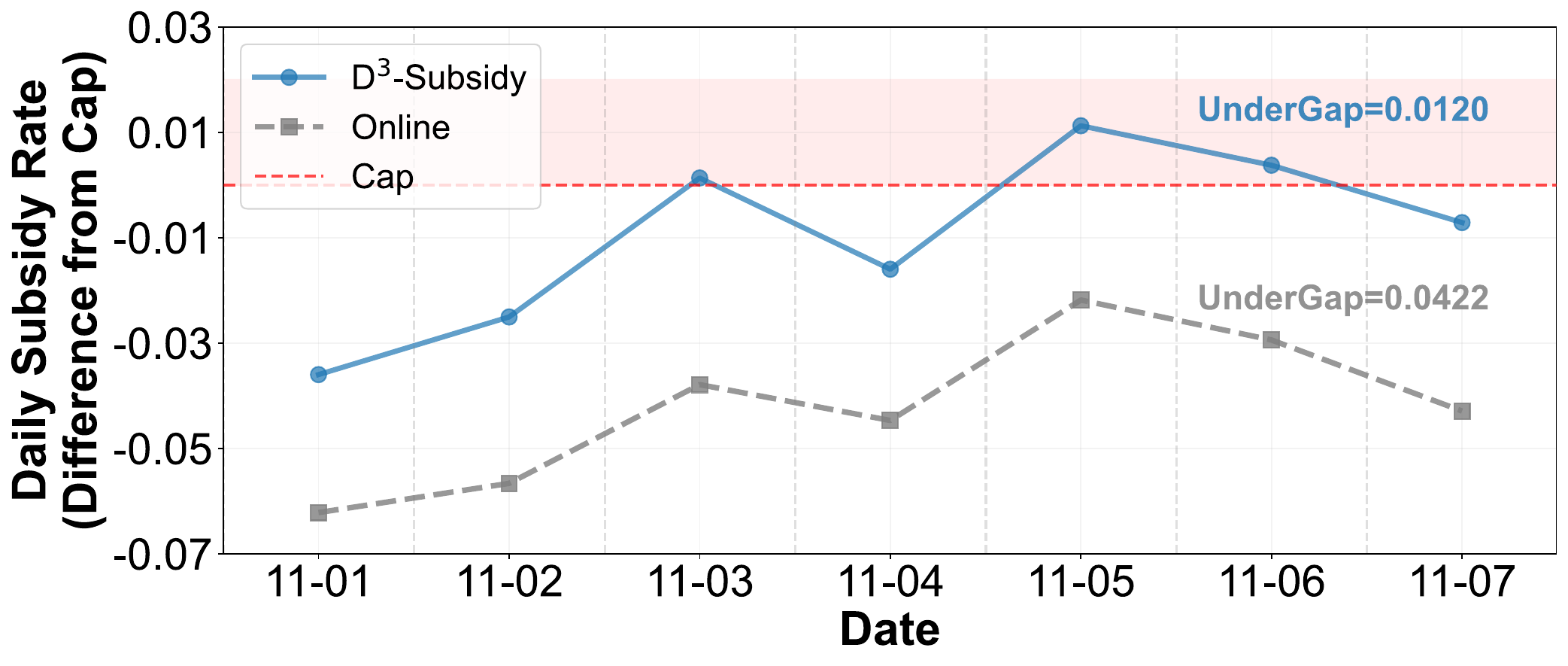}
    \caption{Daily Subsidy Rate in City A.}
    \label{fig:tr_041}
\end{figure}

\begin{figure}[htbp]
    \centering
    \includegraphics[width=0.9\linewidth]{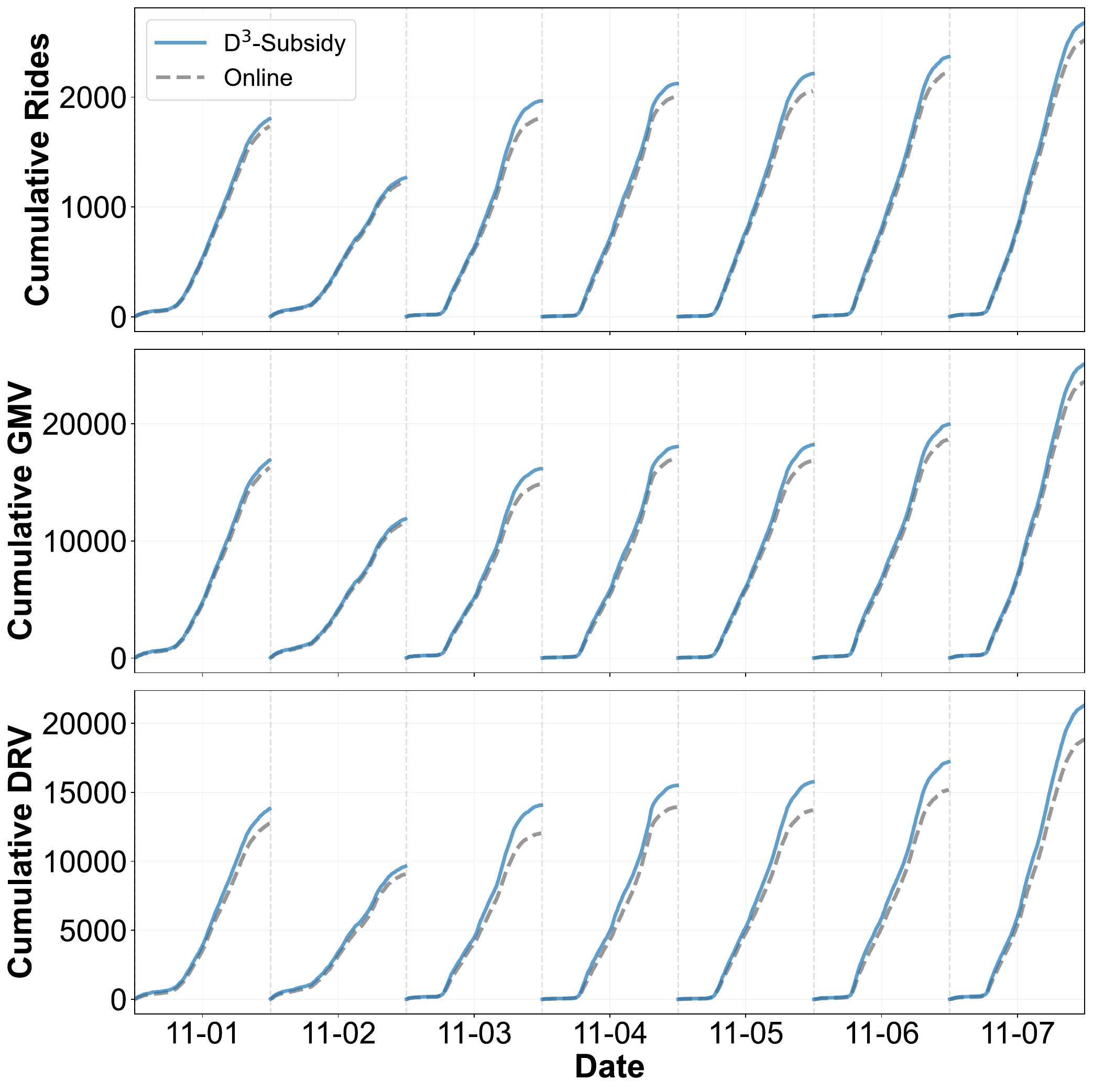}
    \caption{Cumulative \texttt{Rides}, \texttt{GMV} and \texttt{DRV} in City B.}
    \label{fig:cumulative_213}
\end{figure}

\begin{figure}[htbp]
    \centering
    \includegraphics[width=0.9\linewidth]{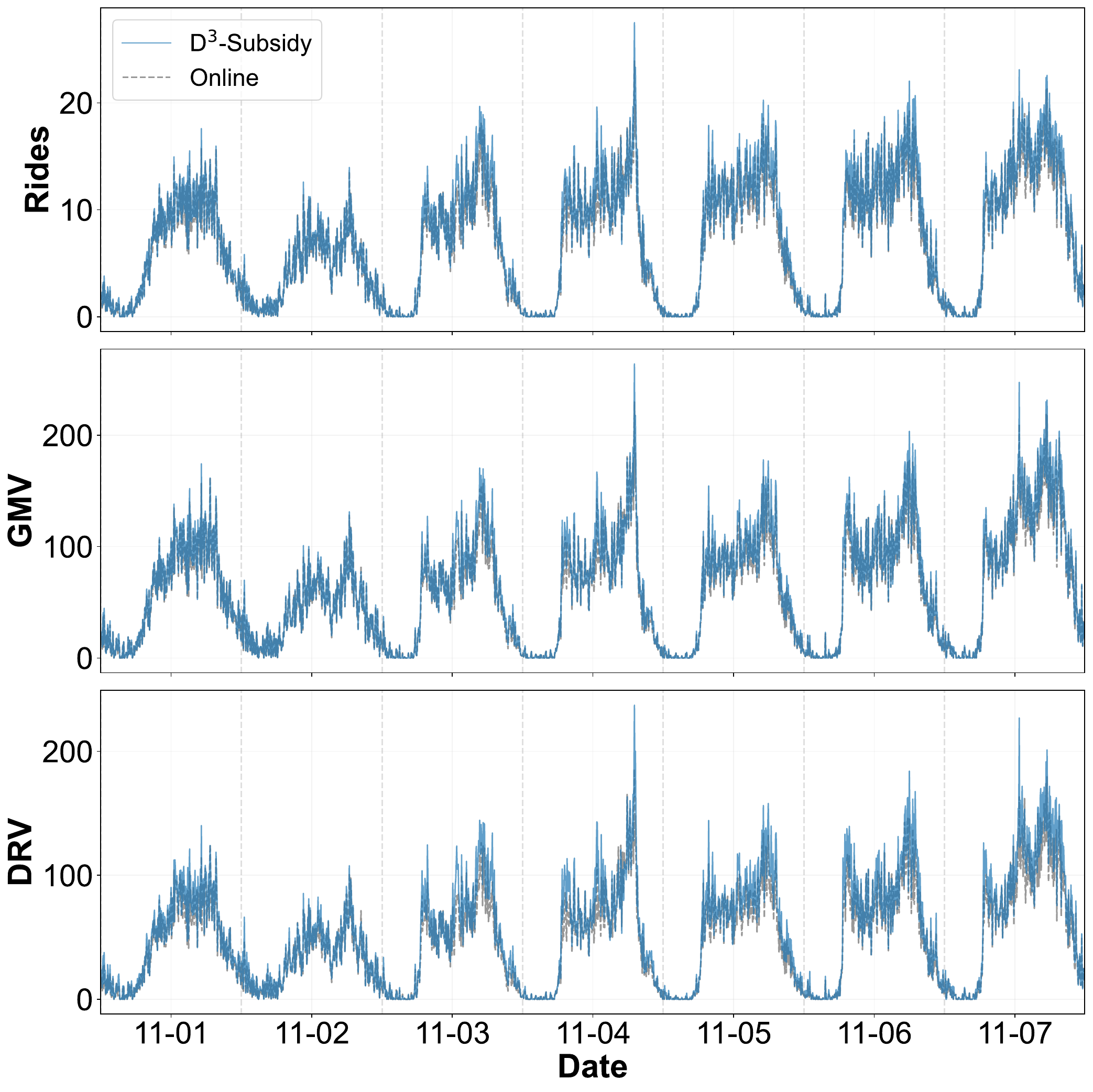}
    \caption{Per-Window \texttt{Rides}, \texttt{GMV} and \texttt{DRV} in City B.}
    \label{fig:per_slot_213}
\end{figure}

\begin{figure}[htbp]
    \centering
    \includegraphics[width=0.92\linewidth]{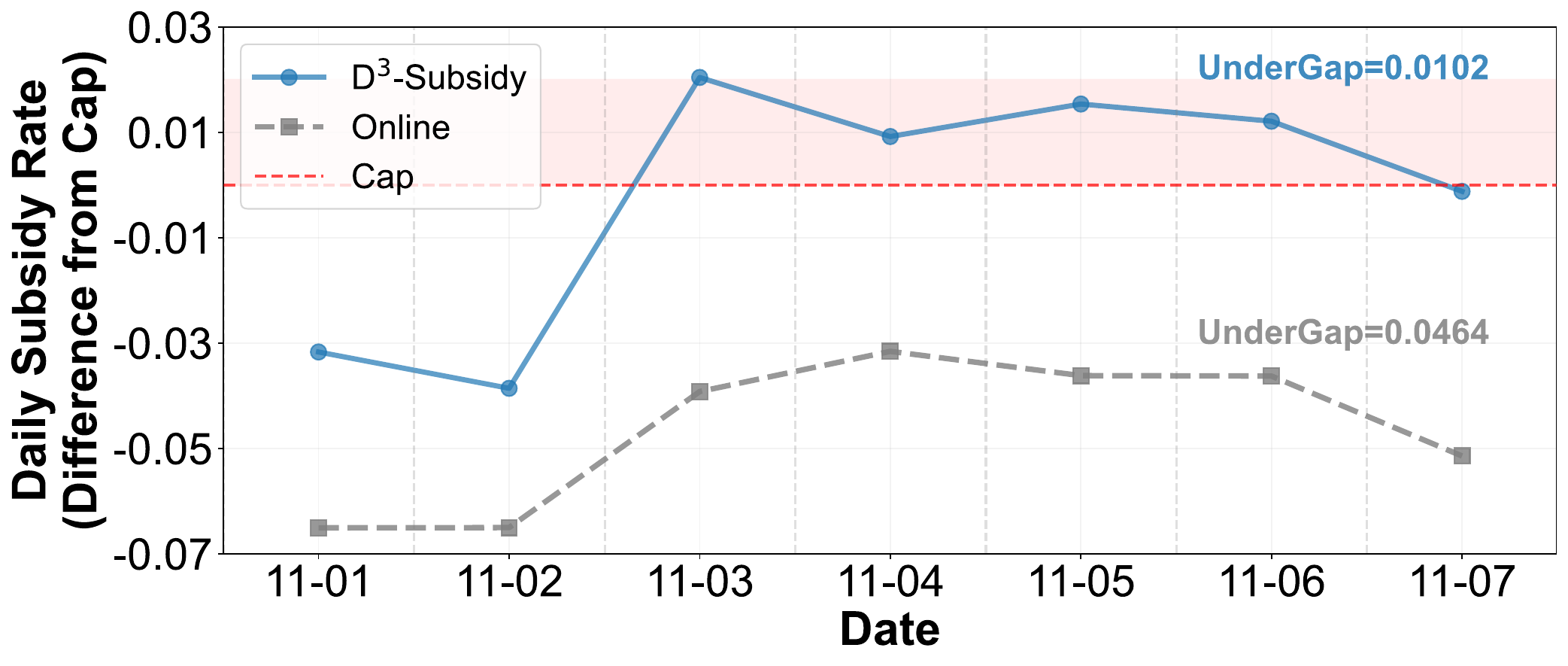}
    \caption{Daily Subsidy Rate in City B.}
    \label{fig:tr_213}
\end{figure}

\clearpage

\end{document}